\documentclass[11pt]{article}
\usepackage{colt11e2}

 \usepackage[numbers]{natbib}
\usepackage{amsmath,amssymb,amsthm,amsfonts,mathtools,color,nicefrac,graphicx,multirow}

\title{Learning with the Weighted Trace-norm under Arbitrary Sampling Distributions}

\usepackage[left=2.8cm,right=2.8cm,top=3cm,bottom=3cm]{geometry}

\author{Rina Foygel \\  Department of Statistics\\ University of Chicago \\ \texttt{rina@uchicago.edu} \\ \And 
Ruslan Salakhutdinov \\ Department of Brain and Cognitive Sciences and CSAIL\\ Massachusetts Institute of Technology \\ \texttt{rsalakhu@mit.edu}\\\AND
Ohad Shamir \\Microsoft Research New England \\ \texttt{ohadsh@microsoft.com}\\ \And
Nathan Srebro\\Toyota Technological Institute at Chicago\\ \texttt{nati@ttic.edu}\\}

\makeatletter
\newtheorem*{rep@theorem}{\rep@title}
\newcommand{\newreptheorem}[2]{%
\newenvironment{rep#1}[1]{%
 \def\rep@title{#2 \ref{##1}}%
 \begin{rep@theorem}}%
 {\end{rep@theorem}}}
\makeatother

\newreptheorem{theorem}{Theorem}
\newreptheorem{lemma}{Lemma}
\newreptheorem{corollary}{Corollary}
\newtheorem{lemma}{Lemma}

\newcommand{\pp}{p}
\newcommand{\p}[2]{\pp\left({#1},{#2}\right)}
\newcommand{\pij}{\p{i}{j}}
\newcommand{\pprow}{\pp^r}
\newcommand{\ppcol}{\pp^c}
\newcommand{\prow}[1]{\pprow\left({#1}\right)}
\newcommand{\pcol}[1]{\ppcol\left({#1}\right)}
\newcommand{\prowi}{\prow{i}}
\newcommand{\pcolj}{\pcol{j}}

\newcommand{\tr}{\mathrm{tr}}
\newcommand{\wtr}{\tr\left(\pprow,\ppcol\right)}
\newcommand{\twtr}{\tr\left(\tpprow,\tppcol\right)}
\newcommand{\hwtr}{\tr\left(\hpprow,\hppcol\right)}
\newcommand{\cwtr}{\tr\left(\cpprow,\cppcol\right)}

\newcommand{\trnorm}[1]{\left\|{#1}\right\|_{\tr}}
\newcommand{\wtrnorm}[1]{\left\|{#1}\right\|_{\wtr}}
\newcommand{\twtrnorm}[1]{\left\|{#1}\right\|_{\twtr}}
\newcommand{\hwtrnorm}[1]{\left\|{#1}\right\|_{\hwtr}}
\newcommand{\cwtrnorm}[1]{\left\|{#1}\right\|_{\cwtr}}

\newcommand{\frnorm}[1]{\left\|{#1}\right\|_F}
\newcommand{\wfrnorm}[1]{\left\|{#1}\right\|_{F\left(\pprow,\ppcol\right)}}

\newcommand{\tpp}{\tilde{p}}

\newcommand{\tpprow}{\tpp^r}
\newcommand{\tppcol}{\tpp^c}
\newcommand{\tprow}[1]{\tpprow\left({#1}\right)}
\newcommand{\tpcol}[1]{\tppcol\left({#1}\right)}
\newcommand{\tprowi}{\tprow{i}}
\newcommand{\tpcolj}{\tpcol{j}}

\newcommand{\hpp}{\hat{p}}

\newcommand{\hpprow}{\hpp^r}
\newcommand{\hppcol}{\hpp^c}
\newcommand{\hprow}[1]{\hpprow\left({#1}\right)}
\newcommand{\hpcol}[1]{\hppcol\left({#1}\right)}

\newcommand{\hprowi}{\hprow{i}}
\newcommand{\hpcolj}{\hpcol{j}}

\newcommand{\cpp}{\check{p}}

\newcommand{\cpprow}{\cpp^r}
\newcommand{\cppcol}{\cpp^c}
\newcommand{\cprow}[1]{\cpprow\left({#1}\right)}
\newcommand{\cpcol}[1]{\cppcol\left({#1}\right)}
\newcommand{\cprowi}{\cprow{i}}
\newcommand{\cpcolj}{\cpcol{j}}

\newcommand{\opp}{\overline{p}}

\newcommand{\opprow}{\opp^r}
\newcommand{\oppcol}{\opp^c}
\newcommand{\oprow}[1]{\opprow\left({#1}\right)}
\newcommand{\opcol}[1]{\oppcol\left({#1}\right)}
\newcommand{\oprowi}{\oprow{i}}
\newcommand{\opcolj}{\opcol{j}}

\newcommand{\R}{\mathbb{R}}

\newcommand{\infnorm}[1]{\left|{#1}\right|_{\infty}}
\newcommand{\indicator}[1]{\mathbb{I}\left\{{#1}\right\}}
\newcommand{\Zcal}{\mathcal{Z}}
\newcommand{\Xcal}{\mathcal{X}}
\newcommand{\W}[1]{\mathcal{W}_r\left[{#1}\right]}
\newcommand{\Wcal}{\W{\pp}}
\newcommand{\tWcal}{\W{\tpp}}

\newcommand{\cWcal}{\W{\cpp}}
\newcommand{\oWcal}{\W{\opp}}

\newcommand{\Diag}[1]{\mathrm{diag}\left({#1}\right)}
\newcommand{\Ical}{\mathcal{I}}

\newcommand{\rank}[1]{\mathrm{rank}\left({#1}\right)}

\newcommand{\Rcal}{\mathcal{R}}

\newcommand{\reals}{\mathbb{R}}
\newtheorem{theorem}{Theorem}
\newcommand{\norm}[1]{\left\lVert{#1}\right\rVert}

\newcommand{\empRad}{\hat{\Rcal}}
\newcommand{\EE}[1]{{{\mbox{\bf E}}\left[{#1}\right]}}
\newcommand{\Ep}[2]{{{\mbox{\bf E}}_{#1}\left[{#2}\right]}}
\newcommand{\spnorm}[1]{\norm{#1}_{\mathrm{sp}}}

\begin{document}

\maketitle

\begin{abstract}
  We provide rigorous guarantees on learning with the weighted
  trace-norm under arbitrary sampling distributions.  We show that
   the standard weighted trace-norm might fail
   when the sampling distribution is not a product distribution
  (i.e.~when row and column indexes are not selected independently), present a corrected
  variant for which we establish strong learning guarantees, and
  demonstrate that it works better in practice.  We provide guarantees
  when weighting by either the true or empirical sampling
  distribution, and suggest that even if the true distribution is
  known (or is uniform), weighting by the empirical distribution may
  be beneficial.
\end{abstract}

\section{Introduction}

One of the most common approaches to collaborative filtering and
matrix completion is trace-norm regularization
\citep{SrebRenJaa04,SalaMni07,Bach08,CanTao09}.  In this approach we
attempt to complete an unknown matrix, based on a small subset of
revealed entries, by finding a matrix with small trace-norm, which
matches those entries as best as possible.

This approach has repeatedly shown good performance in practice, and
is theoretically well understood for the case where revealed entries
are sampled
uniformly~\citep{ShraibmanSrebro,Recht,KMO,Koltchinskii,NW,RinaNatiCOLT}.
Under such uniform sampling, $\Theta(n\log (n))$ entries are sufficient for
good completion of an $n \times n$ matrix---i.e.~a nearly constant
number of entries per row.  However, for arbitrary sampling
distributions, the worst-case sample complexity lies between a lower
bound of $\mathbf{\Omega}(n^{4/3})$~\citep{RusAndNatiNIPS} and an upper bound
of $\mathbf{O}(n^{3/2})$~\citep{S_SS_COLT}, i.e.~requiring between $n^{1/3}$
and $n^{\nicefrac{1}{2}}$ observations per row, and indicating it is not
appropriate for matrix completion in this setting.

Motivated by these issues, Salakhutdinov and Srebro
\citep{RusAndNatiNIPS} proposed to use a weighted variant of the
trace-norm, which takes the distribution of the entries into account,
and showed experimentally that this variant indeed leads to superior
performance.  However, although this recent paper established that the
weighted trace-norm corrects a specific situation where the standard
trace-norm fails, no general learning guarantees are provided, and it
is not clear if indeed the weighted trace-norm always leads to the
desired behavior.  The only theoretical analysis of the weighted
trace-norm that we are aware of is a recent report by Negahban and Wainwright
\citep{NW} that provides reconstruction guarantees for a
low-rank matrix with i.i.d.~noise, but only when the sampling
distribution is a {\em product distribution}, i.e.~the rows index and
column index of observed entries are selected independently.  A
product distribution assumption does not seem realistic in many
cases---e.g.~for the Netflix data, it would indicate that all users have the same
(conditional) distribution over which movies they rate.

In this paper we rigorously study learning with a weighted trace-norm
under an {\em arbitrary} sampling distribution, and show that this
situation is indeed more complicated, requiring a correction to the
weighting.  We show that this correction is necessary, and present
empirical results on the Netflix and MovieLens dataset indicating that
it is also helpful in practice.  We also rigorously consider weighting
according to either the true sampling distribution (as in
\citep{NW}) or the empirical frequencies, as is actually
done in practice, and present evidence that weighting by the empirical
frequencies might be advantageous.  Our setting is also more general
then that of~\citep{NW}---we consider an arbitrary loss
and do not rely in i.i.d.~noise, instead presenting results in an
agnostic learning framework.

\paragraph{Setup and Notation.} We consider an arbitrary unknown $n\times m$ target matrix $Y$, where
a subset of entries $\{Y_{i_t,j_t}\}_{t=1}^{s}$ indexed by
$S=\left\{(i_1,j_1),\dots,(i_s,j_s)\right\}$ is revealed to us. Without loss of generality, we assume $n\geq m$.
Throughout most of the paper, we assume $S$ is drawn i.i.d.~according
to some sampling distribution $p(i,j)$ (with replacement).  Based on
this subset on entries, we would like to fill in the missing entries
and obtain a prediction matrix $\hat{X}_S \in \R^{n \times m}$, with low
expected loss $L_p(\hat{X}_S)=\Ep{ij\sim p}{\ell((\hat{X}_S)_{ij},Y_{ij})}$,
where $\ell(x,y)$ is some loss function.  Note that
 we measure the loss with respect to the same distribution
$p(i,j)$ from which the training set is drawn (this is also the case
in~\citep{RusAndNatiNIPS,NW,S_SS_COLT}).

Given some distribution $p(i,j)$ on $[n]\times[m]$, the weighted
trace-norm of a matrix $X\in\R^{n\times m}$ is given by
\citep{RusAndNatiNIPS}
$$\wtrnorm{X}=\trnorm{\Diag{\pprow}^{\nicefrac{1}{2}} \cdot X \cdot
  \Diag{\ppcol}^{\nicefrac{1}{2}}}\;,$$ where $\pprow\in\R^n$ and $\ppcol\in\R^m$
denote vectors of the row- and column-marginals respectively.  Note
that the weighted trace-norm only depends on these marginals (but not
their joint distribution) and that if $\pprow$ and $\ppcol$ are
uniform, then $\wtrnorm{X}=\frac{1}{\sqrt{nm}}\trnorm{X}$. The
weighted trace-norm does not generally scale with $n$ and $m$, and in
particular, if $X$ has rank $r$ and entries bounded in $[-1,1]$, then
$\wtrnorm{X}\leq \sqrt{r}$ regardless of which $\pij$ is used.  This
motivates us to define the class
$$\Wcal = \{X\in \reals^{n\times m}:\wtrnorm{X}\leq \sqrt{r}\},$$
although we emphasize that our results do
not directly depend on the rank, and $\Wcal$ certainly includes
full-rank matrices.  We analyze here estimators of the form
$\hat{X}_S=\arg\min\{\hat{L}_S(X) \ : \ X\in \Wcal\}$ where
$\hat{L}_S(X)=\frac{1}{s}\sum_{t=1}^s\ell(X_{i_t,j_t},Y_{i_t,j_t})$
is the empirical error on the observed entries.

Although we focus mostly on the standard inductive setting, where the
samples are drawn i.i.d.~and the guarantee is on generalization for
future samples drawn by the same distribution, our results can also be
stated in a transductive model, where a training set and a test set
are created by splitting a fixed subset of entries uniformly at
random (as in \citep{S_SS_COLT}).  The transductive setting is
discussed in Section 4.2, and variants of our Theorems in this setting
are found there and in Appendix \ref{AppendixTransductive}.

\section{Learning with the Standard Weighting}

In this Section, we consider learning using the weighted trace-norm as
suggested by Salakhutdinov and Srebro~\citep{RusAndNatiNIPS}, i.e.~when the weighting
is according to the sampling distribution $p(i,j)$.  Following the
approach of~\citep{ShraibmanSrebro} and~\citep{RinaNatiCOLT}, we base
our results on bounding the Rademacher complexity of $\Wcal$, as a
class of functions mapping index pairs to entry values.  However, we
modify the analysis for the weighted trace-norm with non-uniform
sampling.

For a class of matrices $\Xcal$ and a sample
$S=\left\{(i_1,j_1),\dots,(i_s,j_s)\right\}$ of indexes in $[n]\times
[m]$, the empirical Rademacher complexity of the class (with respect
to $S$) is given by
$$\hat{\Rcal}_S(\Xcal)=\Ep{\sigma\sim\{\pm1\}^s}{\sup_{X\in\Xcal}\frac{1}{s}\sum_{t=1}^s \sigma_t X_{i_tj_t}}\;,$$
where $\sigma$ is a vector of signs drawn uniformly at random.
Intuitively, $\hat{\Rcal}_S(\Xcal)$ measures the extent to which the
class $\Xcal$ can ``overfit'' data, by finding a matrix $X$ which
correlates as strongly as possible to a sample from a matrix of random noise.  For a
loss $\ell(x,y)$ that is Lipschitz in $x$, the Rademacher complexity
can be used to uniformly bound the deviations $|L_p(X)-\hat{L}_S(X)|$
for all $X\in\Xcal$, yielding a learning guarantee on the empirical
risk minimizer~\citep{BartlettMendelson}.

\subsection{Guarantees for Special Sampling Distributions}

We begin by providing guarantees for an arbitrary, possibly unbounded,
Lipschitz loss $\ell(x,y)$, but only under sampling distributions
which are {\em either} product distributions (i.e.~$p(i,j)=p^r(i)
p^c(j)$) {\em or} have uniform marginals (i.e.~$p^r$ and $p^c$ are
uniform, but perhaps the rows and columns are not independent).  In
Section~\ref{SmoothingExamples} below, we will see why this severe
restriction on $\pp$ is needed.

\begin{theorem}\label{SquareRootThm}
  For an $l$-Lipschitz loss $\ell$, fix any matrix $Y$, sample size
  $s$, and distribution $\pp$, such that $\pp$ is either a product
  distribution or has uniform marginals.

  Let
  $\hat{X}_S=\arg\min\left\{\hat{L}_S(X) \ : \
   X\in\Wcal\right\}$. Then, in expectation over
  the training sample $S$ drawn i.i.d. from the distribution $\pp$,
\begin{equation}\label{SquareRootEq}L_p(\hat{X}_S)\leq
  \inf_{X\in\Wcal}L_p(X)+\mathbf{O}\left(l \cdot \sqrt{\frac{rn\log(n)}{s}}\right)\;.\end{equation}
\end{theorem}
Here and elsewhere we state learning guarantees in expectation for
simplicity, but all guarantees can also be obtained with high probability.

\begin{proof} We will show how to bound the expected
  Rademacher complexity $\Ep{S}{\hat{\Rcal}_S(\Wcal)}$, from which the
  desired results follows using standard arguments \citep{BartlettMendelson}.

  Following \citep{RinaNatiCOLT} by including the weights, using the duality between spectral norm
  $\spnorm{\cdot}$ and trace-norm, we compute:
  \begin{align*}
   & \Ep{S}{\empRad_S(\Wcal) }=
\frac{\sqrt{r}}{s} \Ep{S,\sigma}{\spnorm{
      \sum_{t=1}^s \sigma_t \frac{e_{i_t,j_t}}{\sqrt{\prow{i_t}\pcol{j_t}}}}} = \frac{\sqrt{r}}{s} \Ep{S,\sigma}{\spnorm{
      \sum_{t=1}^s Q_t}}\;,
    \end{align*}
where $e_{i,j}=e_ie_j^T$ and $Q_t = \sigma_t \frac{e_{i_t,j_t}}{\sqrt{\prow{i_t}\pcol{j_t}}}
\in \R^{n \times m}$. Since the $Q_t$'s are i.i.d. zero-mean matrices, Theorem 6.1 of~\citep{TroppTailBounds}, combined with Remarks 6.4 and
6.5 there, establishes that
\begin{equation*}
  \Ep{S,\sigma}{\spnorm{\sum_{t=1}^s Q_t}} = \mathbf{O}\left( \sigma \sqrt{\log(n)} + R \log(n) \right)\;,
\end{equation*}
where $\spnorm{Q_t}\leq R$ (almost surely) and $\sigma^2 =
\max\left\{\spnorm{\sum\EE{Q_t^T Q_t}},\spnorm{\sum\EE{Q_t
      Q^T_t}} \right\}$.  Calculating these (see  Appendix~\ref{AppendixProofs}
), we get $R\leq\sqrt{\frac{nm}{\min_{i,j}\{n\prowi\cdot m\pcolj\}}}$, and
      $$\sigma\leq  \sqrt{s\max\left\{\max_i\sum_j\frac{\pij}{\prowi\pcolj},\max_j\sum_i\frac{\pij}{\prowi\pcolj}\right\}}\leq \sqrt{\frac{sn}{\min_{i,j}\{n\prowi\cdot m\pcolj\}}}\;.$$

If $\pp$ has uniform row- and column-marginals, then for all $i,j$, $n\prowi=m\pcolj=1$. This yields
$$\Ep{S}{\hat{\Rcal}(\Wcal)}\leq \mathbf{O}\left( \sqrt{\frac{rn\log(n)}{s}}\right)\;,$$
as desired. (Here we assume $s>n\log(n)$, since otherwise we need
only establish that excess error is $O(l \sqrt{r})$, which holds
trivially for any matrix in $\Wcal$.)

If $\pp$ does not have uniform marginals, but instead is a product distribution, then the quantity $R$ defined above is potentially unbounded, so we cannot apply the same simple argument. However, we can consider the ``$\pp$-truncated'' class of matrices
$$\Zcal=\left\{Z(X)=\left(X_{ij}\indicator{\pij\geq\frac{\log(n)}{s\sqrt{nm}}}\right)_{ij}:X\in\Wcal\right\}\;.$$
By a similar calculation of the expected spectral norms, we can now bound
$\Ep{S}{\hat{\Rcal}_S(\Zcal)}\leq \mathbf{O}\left( \sqrt{\frac{rn\log(n)}{s}}\right)$.
Applying~\citep{BartlettMendelson}, this bounds $\left(L_p(Z(\hat{X}_S))-\hat{L}_S(Z(\hat{X}_S))\right)$ (in expectation). Since $Z(\hat{X}_S)_{ij}\neq(\hat{X}_S)_{ij}$ only on the extremely low-probability entries, we can also bound $\left(L_p(\hat{X}_S)-L_p(Z(\hat{X}_S))\right)$ and $\left(\hat{L}_S(Z(\hat{X}_S))-\hat{L}_S(\hat{X}_S)\right)$. Combining these steps, we can bound $\left(L_p(\hat{X}_S)-\hat{L}_S(\hat{X}_S)\right)$. We similarly  bound $\hat{L}_S(X^*)-L_p(X^*)$, where $X^*=\arg\min_{X\in\Wcal}L_p(X)$. Since $\hat{L}_S(\hat{X}_S)\leq\hat{L}_S(X^*)$, this yields the desired bound on excess error. The details are given in  Appendix~\ref{AppendixProofs}.
\end{proof}

Examining the proof of Theorem~\ref{SquareRootThm}, we see that we can
generalize the result by including distributions $\pp$ with row- and
column-marginals that are lower-bounded. More precisely, if $\pp$
satisfies $\prowi\geq\frac{1}{Cn}$, $\pcolj\geq\frac{1}{Cm}$ for all
$i,j$, then the bound~(\ref{SquareRootEq}) holds, up to a factor of
$C$. Note that this result does not require an upper bound on the row-
and column-marginals, only a lower bound, i.e.~it only requires that
no marginals are too low.  This is important to note since the
examples where the unweighted trace-norm fails under a non-uniform
distribution are situations where some marginals are very {\em high}
(but none are too low) ~\citep{RusAndNatiNIPS}.  This suggests that
the low-probability marginals could perhaps be ``smoothed'' to satisfy
a lower bound, without removing the advantages of the weighted
trace-norm. We will exploit this in Section~\ref{SmoothingSection} to
give a guarantee that holds more generally for arbitrary $\pp$, when
smoothing is applied.

\subsection{Guarantees for bounded loss}

In Theorem~\ref{SquareRootThm}, we showed a strong bound on excess
error, but only for a restricted class of distributions $\pp$. We now
show that if the loss function $\ell$ is bounded, then we can give a
non-trivial, but weaker, learning guarantee that holds uniformly over
all distributions $\pp$.  Since we are in any case discussing
Lipschitz loss functions, requiring that the loss function be bounded
essentially amounts to requiring that the entries of the matrices
involved be bounded.  That is, we can view this as a guarantee on
learning matrices with bounded entries.  In
Section~\ref{SmoothingExamples} below, we will show that this
boundedness assumption is unavoidable if we want to give a guarantee
that holds for arbitrary $\pp$.

\begin{theorem}\label{CubeRootThm}
  For an $l$-Lipschitz loss $\ell$ bounded by $b$, fix any matrix $Y$,
  sample size $s$, and any distribution $\pp$.  Let
  $\hat{X}_S=\arg\min\left\{\hat{L}_S(X) \ : \ X\in\Wcal\right\}$ for $r\geq 1$. Then, in expectation over the
  training sample $S$ drawn i.i.d. from the distribution $\pp$,
\begin{equation}\label{CubeRootEq}L_p(\hat{X}_S)\leq
  \inf_{X\in\Wcal}L_p(X)+\mathbf{O}\left( (l+b)\cdot \sqrt[3]{\frac{rn\log(n)}{s}}\right)\;.\end{equation}
\end{theorem}
The proof is provided in  Appendix~\ref{AppendixProofs}, and is again
based on analyzing the expected Rademacher complexity,
$\Ep{S}{\hat{\Rcal}(\ell\circ\Wcal)}\leq \mathbf{O}\left((l+b)\cdot \sqrt[3]{\frac{rn\log(n)}{s}}\right)$.

\subsection{Problems with the standard weighting}\label{SmoothingExamples}
In the previous Sections, we showed that for distributions $\pp$ that are either product distributions or have uniform marginals, we can prove a square-root bound on excess error, as shown in~(\ref{SquareRootEq}). For arbitrary $\pp$, the only learning guarantee we obtain is a cube-root bound given in~(\ref{CubeRootEq}), for the special case of bounded loss. We would like to know whether the square-root bound might hold uniformly over all distributions $\pp$, and if not, whether the cube-root bound is the strongest result that we can give in this case for the bounded-loss setting, and whether any bound will hold uniformly over all $\pp$ in the unbounded-loss setting.

The examples below demonstrate that we cannot improve the results of
Theorems~\ref{SquareRootThm} and~\ref{CubeRootThm} (up to log
factors), by constructing degenerate examples using non-product
distributions $\pp$ with non-uniform marginals. Specifically, in
Example~1, we show that in the special
case of bounded loss, the cube-root bound in~\ref{CubeRootEq} is the
best possible bound (up to the log factor) that will hold for all
$\pp$, by giving a construction for arbitrary $n=m$ and arbitrary
$s\leq nm$, such that with $1$-bounded loss, excess error is
${\mathbf{\Omega}}\left(\sqrt[3]{\frac{n}{s}}\right)$. In
Example~2, we show that with unbounded
(Lipschitz) loss, we cannot bound excess error better than a constant
bound, by giving a construction for arbitrary $n=m$ and arbitrary
$s\leq nm$ in the unbounded-loss regime, where excess error is
$\mathbf{\Omega}(1)$. For both examples we fix $r=1$. We note that both examples can be modified to fit the
transductive setting, demonstrating that smoothing is necessary also in
the transductive setting as well.

{\bf Example 1.}
Let $\ell(x,y)=\min\{1,|x-y|\}\leq 1$, let $a=(2s/n)^{\nicefrac{2}{3}}<n$, and let matrix $Y$ and block-wise constant distribution $\pp$ be given by
$$Y=\left(\begin{array}{cc}A&\mathbf{0}_{a\times\tfrac{n}{2}}\\\mathbf{0}_{(n-a)\times\tfrac{n}{2}}&\mathbf{0}_{(n-a)\times\tfrac{n}{2}}\\\end{array}\right)\;,  \ \ (\pij)=\left(\begin{array}{cc}\frac{1}{2s}\cdot\mathbf{1}_{a\times \tfrac{n}{2}}&\mathbf{0}_{a\times\tfrac{n}{2}}\\\mathbf{0}_{(n-a)\times\tfrac{n}{2}}&\frac{1-\tfrac{an}{4s}}{(n-a)\frac{n}{2}}\cdot\mathbf{1}_{(n-a)\times\tfrac{n}{2}}\\\end{array}\right)\;,$$
where $A\in\{\pm 1\}^{a\times\tfrac{n}{2}}$ is any sign matrix.
Clearly, $\wtrnorm{Y}\leq 1$, and so $\inf_{X\in\Wcal}L_p(X)=0$.
Now suppose we draw a sample $S$ of size $s$ from the matrix $Y$,
according to the distribution $\pp$.  We will show an ERM $\hat{Y}$
such that in expectation over $S$, $L_p(\hat{Y})\geq
\frac{1}{8}\sqrt[3]{\frac{n}{s}}$.

Consider $Y^S$ where $Y^S_{ij}=Y_{ij}\indicator{ij\in S}$, and note
that $\wtrnorm{Y^S}\leq 1$. Since $\hat{L}_S(Y^S)=0$, it
clearly an ERM.  We also have $L_p(Y^S)=\frac{N}{2s}$, where $N$ is
the number of $\pm1$'s in $Y$ which are not observed in the sample.
Since $\EE{N}\geq\frac{an}{4}$, we see that $\EE{L_p(Y^S)}\geq
\frac{1}{2s}\cdot \frac{an}{4}\geq\frac{1}{8}\sqrt[3]{\frac{n}{s}}$.

{\bf Example 2.} Let $\ell(x,y)=|x-y|$. Let $Y=\mathbf{0}_{n\times n}$; trivially, $Y\in\Wcal$. Let $\p{1}{1}=\frac{1}{s}$, and $\p{i}{1}=\p{1}{j}=0$ for all $i,j>1$, yielding $\prow{1}=\pcol{1}=\frac{1}{s}$. (The other entries of $\pp$ may be defined arbitrarily.) We will show an ERM $\hat{Y}$ such that, in expectation over $S$, $L_p(\hat{Y})\geq 0.25$. Let $A$  be the matrix with $X_{11}=s$ and zeros elsewhere, and note that $\wtrnorm{A}=1$. With probability $\geq 0.25$, entry $(1,1)$ will not appear in $S$, in which case $\hat{Y}=A$ is an ERM, with $L_p(\hat{Y})=1$.
\vspace{.3cm}

The following table summarizes the learning
guarantees that can be established for the (standard) weighted
trace-norm.  As we saw, these guarantees are tight up to log-factors.

\begin{tabular}{|c|c|c|}\hline
&$1$-Lipschitz, $1$-bounded loss&$1$-Lipschitz, unbounded loss\\\hline
$\pp=$ product&$ \sqrt{\frac{rn\log(n)}{s}}$&$ \sqrt{\frac{rn\log(n)}{s}}$\\\hline
$\pprow,\ppcol=$ uniform&$ \sqrt{\frac{rn\log(n)}{s}}$&$ \sqrt{\frac{rn\log(n)}{s}}$\\\hline
$\pp$ arbitrary&$ \sqrt[3]{\frac{rn\log(n)}{s}}$&1\\\hline
\end{tabular}

\section{Smoothing the weighted trace norm}\label{SmoothingSection}

Considering Theorem~\ref{SquareRootThm} and the degenerate examples in
Section~\ref{SmoothingExamples}, it seems that in order to be able to
generalize for non-product distributions, we need to enforce some sort
of uniformity on the weights.  The Rademacher complexity computations
in the proof of Theorem~\ref{SquareRootThm} show that the problem lies
not with large entries in the vectors $\pprow$ and $\ppcol$ (i.e. if
$\pprow$ and/or $\ppcol$ are ``spiky''), but with the small entries in
these vectors. This suggests the possibility of ``smoothing'' any
overly low row- or column-marginals, in order to improve learning
guarantees.

In Section~\ref{SmoothingSec1}, we present such a smoothing, and
provide guarantees for learning with a smoothed weighted
trace-norm.  The result suggests that there is no strong negative
consequence to smoothing, but there might be a large advantage, if
confronted with situations as in Examples~1 and~2.  In
Section \ref{SmoothingSec2} we check the smoothing correction to the
weighted trace-norm on real data, and observe that indeed it can also
be beneficial in practice.

\subsection{Learning guarantee for arbitrary distributions}\label{SmoothingSec1}

Fix a distribution $\pp$ and a constant $\alpha\in(0,1)$, and let $\tpp$ denote the smoothed marginals:
\begin{equation}\label{Smoothing_alpha}\tprowi=\alpha\cdot\prowi+(1-\alpha)\cdot\tfrac{1}{n}, \ \  \ \tpcolj=\alpha\cdot\pcolj+(1-\alpha)\cdot\tfrac{1}{m}\;.\end{equation}
In the theoretical results below, we use $\alpha=\frac{1}{2}$, but up to a constant factor, the same results hold for any fixed choice of $\alpha\in(0,1)$.

\begin{theorem}\label{SmoothingThm}
 For an $l$-Lipschitz loss $\ell$, fix any matrix $Y$,
  sample size $s$, and any distribution $\pp$.
  Let $\hat{X}_S=\arg\min\left\{\hat{L}_S(X) \ : \ X\in\tWcal\right\}$. Then, in expectation over the training sample $S$ drawn i.i.d. from the distribution $\pp$,
\begin{equation}\label{SquareRootEqSmooth}L_p(\hat{X}_S)\leq \inf_{X\in\tWcal}L_p(X)+\mathbf{O}\left(l\cdot \sqrt{\frac{rn\log(n)}{s}}\right)\;.\end{equation}
\end{theorem}
\begin{proof}
We bound $\Ep{S\sim p}{\hat{\Rcal}_S(\tWcal)}\leq \mathbf{O}\left(\sqrt{\frac{rn\log(n)}{s}}\right)$, and then apply~\citep{BartlettMendelson}. The proof of this Rademacher bound is essentially identical to the proof in Theorem~\ref{SquareRootThm}, with the modified definition of $Q_t = \sigma_t \frac{e_{i_t,j_t}}{\sqrt{\tprowi\tpcolj}}$.  Then $\spnorm{Q_t}\leq\max_{ij}\frac{1}{\sqrt{\tprowi\tpcolj}}\leq 2\sqrt{nm}\doteq R$, and
$\EE{\spnorm{\sum_{t=1}^s Q_tQ_t^T}}=s\cdot{\max_i\sum_j\frac{\pij}{\tprowi\tpcolj}}
\leq s\cdot{\max_i\sum_j\frac{\pij}{\frac{1}{2}\prowi\cdot\frac{1}{2m}}}\leq{4sm}$.

Similarly, $\EE{\spnorm{\sum_{t=1}^s Q_t^TQ_t}}\leq{4sn}$. Setting $\sigma\doteq \sqrt{4sn}$ and applying~\citep{TroppTailBounds}, we obtain the result.
\end{proof}

Moving from Theorem~\ref{SquareRootThm} to Theorem~\ref{SmoothingThm},
we are competing with a different class of matrices:
$$\inf_{X\in\Wcal}L_p(X) \rightsquigarrow \inf_{X\in\tWcal}L_p(X)\;.$$
In most applications we can think of, this change is not significant.
For example, we consider the low-rank matrix reconstruction problem,
where the trace-norm bound is used as a surrogate for rank.  In order
for the (squared) weighted trace-norm to be a lower bound on the
rank, we would need to assume $\frnorm{\Diag{\pprow}^{\nicefrac{1}{2}}X\Diag{\ppcol}^{\nicefrac{1}{2}}}^2\leq 1$
\citep{RinaNatiCOLT}. If we also assume that $\norm{(X^*)_{(i)}}^2_2\leq m$ and
$\norm{(X^*)^{(j)}}^2_2\leq n$ for all rows $i$ and columns $j$
--- i.e.~the row and column
magnitudes are not ``spiky'' --- then $X^*\in\tWcal$. Note that this condition is much weaker than
placing a spikiness condition on $X^*$ itself, e.g.~requiring
$\infnorm{X^*}\leq 1$.

\subsection{Results on Netflix and MovieLens Datasets}\label{SmoothingSec2}
We evaluated different models on two publicly-available collaborative filtering datasets: Netflix~\citep{Netflix} and MovieLens~\citep{ML}.
The Netflix dataset consists of 100,480,507 ratings from 480,189 users
on 17,770 movies.
Netflix also
provides qualification set containing 1,408,395 ratings,
but due to the sampling scheme, ratings from users with few ratings are
overrepresented
 relative to the training set.
To avoid dealing with different training and
test distributions,
we also created our own validation and test sets, each containing
100,000 ratings
set aside from the training set. The MovieLens dataset contains 10,000,054 ratings
from 71,567 users and 10,681 movies.
We again
set aside  test and validation sets of 100,000 ratings.
Ratings were normalized to be zero-mean.

When dealing with large datasets the most practical way to
fit trace-norm regularized models is via stochastic gradient descent~\citep{KorenSVDpp,SalaMni07,RusAndNatiNIPS}.
For computational reasons, however, we consider rank-truncated trace-norm minimization, by optimizing
within the restricted class $\{X:X\in\Wcal,\rank{X}\leq k\}$ for $k=30$ and $k=100$, and for  various values of smoothing parameters $\alpha$ (as in~(\ref{Smoothing_alpha})).
For each value of $\alpha$ and $k$, the regularization parameter was chosen by cross-validation.

The following table shows root mean squared error (RMSE) for the experiments. For both k=30 and k=100 the weighted
trace-norm with smoothing significantly outperforms the weighted
trace-norm without smoothing ($\alpha=1$),
 even on the differently-sampled Netflix qualification set.
We also note
that the proposed weighted trace-norm with smoothing outperforms
max-norm regularization \citep{maxnorm}, and compares favorably with
the ``geometric'' smoothing used by \citep{RusAndNatiNIPS} as a
heuristic, without theoretical or conceptual justification.  A moderate value of $\alpha=0.9$ seems consistently
good.

\begin{tabular}{c|ccc|ccc || cc | cc}
\hline
   &           &   \multicolumn{4}{c}{Netflix}  & & \multicolumn{4}{c}{MovieLens} \\
$\alpha$ & k   &  Test & Qual  & k &  Test  & Qual & k & Test & k & Test \\ \hline
1    &  30  & 0.7604  & 0.9107              & 100 & 0.7404   & 0.9078   & 30   & 0.7852        &  100   & 0.7821 \\
0.9  &  30  & 0.7589  & 0.9096              & 100 & 0.7391   & 0.9068   & 30   & 0.7831        &  100  & 0.7798 \\
0.5  &  30  & 0.7601 & 0.9173               & 100 & 0.7419   & 0.9161   & 30   & 0.7836        &  100 & 0.7815 \\
0.3  &  30  & 0.7712  & 0.9198              & 100 & 0.7528   & 0.9207   & 30   & 0.7864        &  100 & 0.7871 \\
0    &  30  & 0.7887  & 0.9249              & 100 & 0.7659   & 0.9236   & 30   & 0.7997        & 100  & 0.7987 \\ \hline

\end{tabular}

\section{The empirically-weighted trace norm}

In practice, the sampling distribution $\pp$ is not known exactly --- it can only be estimated via the locations of the entries which are observed in the sample. Defining the empirical marginals
$$\hprowi=\frac{\#\{t:i_t=i\}}{s},\ \hpcolj=\frac{\#\{t:j_t=j\}}{s}\;,$$
we would like to give a learning guarantee when $\hat{X}_S$ is estimated via regularization on the $\hpp$-weighted trace-norm, rather than the $\pp$-weighted trace-norm.

In Section~\ref{EmpiricalSec1}, we give bounds on excess error when
learning with smoothed empirical marginals, which show that there is
no theoretical disadvantage as compared to learning with the smoothed
true marginals.  In fact, we provide evidence that suggests there
might even be an {\em advantage} to using the empirical marginals.  To
this end, in Section~\ref{EmpiricalSec2}, we introduce the
transductive learning setting, and give a result based on the
empirical marginals which implies a sample complexity bound that is better by
a factor of $\log^{\nicefrac{1}{2}} (n)$.  In Section~\ref{EmpiricalSec3}, we show
that in low-rank matrix reconstruction simulations, using empirical
marginals is indeed yields better reconstructions.

\subsection{Guarantee for the standard (inductive) setting}\label{EmpiricalSec1}

We first show that when learning with the smoothed empirical marginals, defined as
$$\cprowi=\tfrac{1}{2}\left(\hprowi+\tfrac{1}{n}\right), \ \cpcolj=\tfrac{1}{2}\left(\hpcolj+\tfrac{1}{m}\right)\;,$$
we can obtain the same guarantee as for learning with the smoothed (true) marginals, given by $\tpp$.

\begin{theorem}\label{EmpiricalThm} For an $l$-Lipschitz loss $\ell$, fix any matrix $Y$,
  sample size $s$, and any distribution $\pp$.  Let $\hat{X}_S=\arg\min\left\{\hat{L}_S(X) \ : \ X\in\cWcal\right\}$. Then, in expectation over the training sample $S$ drawn i.i.d. from the distribution $p$,
\begin{equation}\label{SquareRootEqEmp}L_p(\hat{X}_S)\leq \inf_{X\in\tWcal}L_p(X)+\mathbf{O}\left(l\cdot \sqrt{\frac{r\max\{n,m\}\log(n+m)}{s}}\right)\;.\end{equation}
\end{theorem}
Note that although we regularize using the (smoothed)
empirically-weighted trace-norm, we still compare ourselves to
the best possible matrix in the class defined by the (smoothed) true
marginals.

The proof of the Theorem (given in  Appendix~\ref{AppendixProofs}) uses
Theorem \ref{SmoothingThm} and involves showing that with a sample of size $s =
\mathbf{\Omega}(n \log (n))$, which is required for all Theorems so far to be
meaningful, the true and empirical marginals are the
same up to a constant factor. For this to be the case, such a sample size is even
necessary.  In fact, the $\log(n)$ factor in our analysis (e.g. in
the proof of Theorem \ref{SquareRootThm}) arises from the bound on the expected
spectral norm of a matrix, which, for a diagonal matrix, is just a
bound on the deviation of empirical frequencies.  Might it be possible, then, to avoid this
logarithmic factor by using the empirical marginals?
Although we
could not establish such a result in the inductive setting,
we now turn to the transductive setting, where we could indeed obtain a
better guarantee.

\subsection{Guarantee for the transductive setting}\label{EmpiricalSec2}

In the transductive model, we fix a set
$\overline{S}\subset[n]\times[m]$ of size $2s$, and then randomly
split $\overline{S}$ into a training set $S$ and a test set $T$ of
equal size $s$.  The goal is to obtain a good estimator for the
entries in $T$ based on the values of the entries in $S$, as well as
the locations (indexes) of all elements on $\overline{S}$. We then use the (smoothed or unsmoothed) empirical marginals of $\overline{S}$, for the weighted trace-norm.

We now show that, for bounded loss, there may be a benefit to
weighting with the smoothed empirical marginals --- the sample size requirement can be lowered to
$s=\mathbf{O}\left(rn\log^{\nicefrac{1}{2}}(n)\right)$.

\begin{theorem}\label{EmpiricalImprovementThm}
For an $l$-Lipschitz loss $\ell$ bounded by $b$, fix any matrix $Y$ and sample size $s$.  Let $\overline{S}\subset[n]\times[m]$ be a fixed subset of size $2s$, split uniformly at random into training and test sets $S$ and $T$, each of size $s$. Let $\opp$ denote the smoothed empirical marginals of $\overline{S}$. Let $\hat{X}_S=\arg\min\left\{\hat{L}_S(X) \ : \ X\in\oWcal\right\}$. Then in expectation over the splitting of $\overline{S}$ into $S$ and $T$,
\begin{equation}\label{SquareRootEqEmpImp}\hat{L}_T(\hat{X}_S)\leq \inf_{X\in\oWcal}\hat{L}_T(X)+\mathbf{O}\left(l\cdot \sqrt{\frac{rn\log^{\nicefrac{1}{2}}(n)}{s}}+\frac{b}{\sqrt{s}}\right)\;.\end{equation}
\end{theorem}

This result (proved in Appendix~\ref{AppendixTransductive}) is stated
in the transductive setting, with a somewhat different sampling
procedure and evaluation criteria, but we believe the main difference
is in the use of the empirical weights.  Although it is usually
straightforward to convert a transductive guarantee to an inductive
one, the situation here is more complicated, since the hypothesis
class depends on the weighting, and hence on the sample
$\overline{S}$.  Nevertheless, we believe such a conversion might be
possible, establishing a similar guarantee for learning with the
(smoothed) empirically weighted trace-norm also in the inductive
setting.  Furthermore, by using the fact that a sample of size
$s=\Theta(n \log(n))$ is sufficient for the empirical marginals to be
close to the true marginals, it might be possible to obtain a learning
guarantee for the true (non-empirical) weighting with a sample of size
$s=\mathbf{O}(n(r \log^{\nicefrac{1}{2}}(n) + \log(n)))$.

Theorem~\ref{EmpiricalImprovementThm} above can be viewed as a
transductive analog to Theorem~\ref{SmoothingThm} (where weights are
based on the combined sample $\overline{S}$).  In Appendix
\ref{AppendixTransductive} we state and prove transductive analogs
also to Theorem~\ref{SquareRootThm} (for the case where smoothing is
not needed) and Theorem~\ref{CubeRootEq} (giving a cubic-root
rate).  As mentioned in Section \ref{SmoothingExamples}, our lower
bound examples can also be stated in the transductive setting, and
thus all our guarantees and lower bounds can also be obtained in this
setting.

\subsection{Simulations with empirical weights}\label{EmpiricalSec3}

In order to numerically investigate the possible advantage of
empirical weighting, we performed simulations on low-rank matrix
reconstruction under uniform sampling with the unweighted, and the
smoothed empirically weighted, trace-norms.  We choose to work with
uniform sampling in order to emphasize the benefit of empirical
weights, even in situations where one might not consider to use any
weights at all.  In all the experiments, we attempt to reconstruct a
possibly noisy, random rank-2 ``signal'' matrix $M$ with
singular values $\tfrac{1}{\sqrt{2}}(n,n,0,\dots,0)$, ensuring
$\frnorm{M}=n$, measuring error using the squared
loss\footnote{Although the squared loss is  Lipschitz in a bounded domain, it is probably
  possible to improve all our results (removing the square root) in
  the special case of the squared loss, possibly with the additional assumption of i.i.d.~noise
 , as in \citep{NW}.}.   Simulations were performed using \textsc{Matlab}, with code adapted from the \textsc{SoftImpute} code developed by~\citep{SoftImpute}. We performed two types of simulations:

\paragraph{Sample complexity comparison in the noiseless setting:} We define $Y=M$, and compute
$\hat{X}_S=\arg\min\left\{\norm{X}:\hat{L}_S(X)=0\right\}$,
where $\norm{X}=\trnorm{X}$ or $=\hwtrnorm{X}$, as appropriate.
 In Figure~\ref{p_hat_results}(a), we plot the average number of samples per row needed to get average squared error (over 100 repetitions) of at most $0.1$, with both uniform weighting and empirical weighting.
\paragraph{Excess error comparison in the noiseless and noisy settings:} We define $Y=M+\nu N$, where noise $N$ has i.i.d. standard normal entries. We compute
$\hat{X}_S=\arg\min\left\{\norm{X}:\hat{L}_S(X)\leq \nu^2\right\}$.
In Figure~\ref{p_hat_results}(b), we plot the resulting average squared error (over 100 repetitions) over a range of sample sizes $s$ and noise levels $\nu$, with both uniform weighting and empirical weighting.

The results from both experiments show a significant benefit to using the empirical marginals.

\begin{figure}[htbp]
\begin{center}
\includegraphics[width=7.5cm]{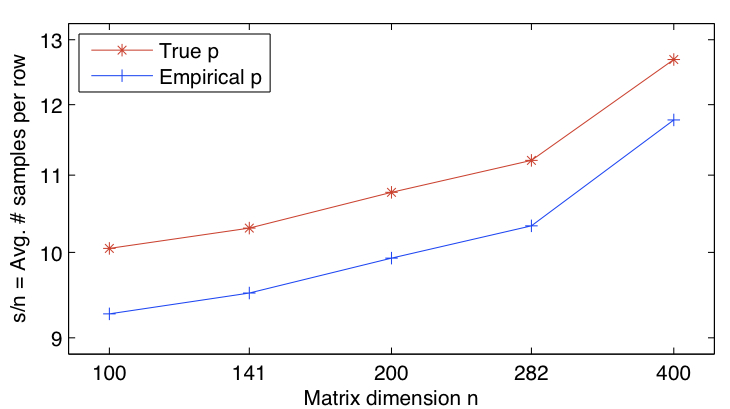}
\includegraphics[width=7.5cm]{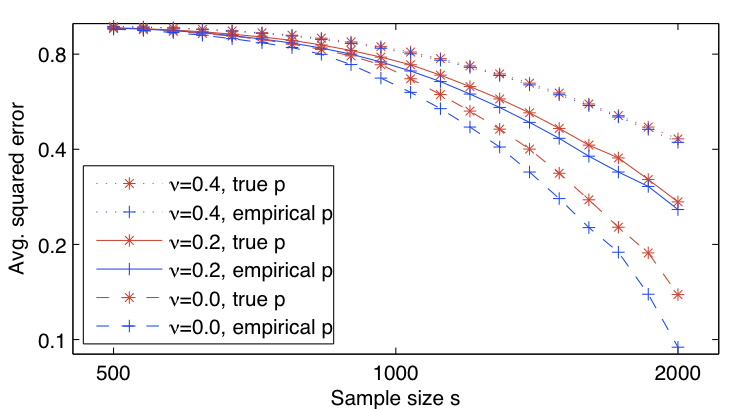}
\caption{\it (a) Left: Sample size needed to obtain avg. error $0.1$, with respect to $n$.  (b) Right: Excess error level over a range of sample sizes, for fixed $n=200$. (Axes are on a logarithmic scale.)}
\label{p_hat_results}
\end{center}
\end{figure}

\section{Discussion}

In this paper, we prove learning guarantees for the weighted
trace-norm by analyzing expected Rademacher complexities. We show that
weighting with smoothed marginals eliminates degenerate scenarios
that can arise in the case of a non-product sampling distribution, and
demonstrate in experiments on the Netflix and MovieLens datasets that
this correction can be useful in applied settings. We also give
results for empirically-weighted trace-norm regularization, and see
indications that using the empirical distribution may
be better than using the true distribution, even if it is available.

\bibliography{WTNbib}

\appendix

\section{Proofs for the i.i.d. sampling setting}\label{AppendixProofs}
\subsection{Proof of Theorem~\ref{SquareRootThm}}
We first fill in the details for the Rademacher bound in the case that $\pp$ has uniform row- and column-marginals. Define
$$Q_t = \sigma_t \frac{e_{i_t,j_t}}{\sqrt{\prow{i_t}\pcol{j_t}}}
\in \R^{n \times m}\;.$$
We need to calculate $R$ and $\sigma^2$ such that $\spnorm{Q_t}\leq R$ (almost surely) and $$\sigma^2 =
\max\left\{\spnorm{\sum\EE{Q_t^T Q_t}},\spnorm{\sum\EE{Q_t
      Q^T_t}} \right\}.$$
For each $t$, $Q_t$ is just a matrix with a single non-zero entry of
magnitude $\frac{1}{\sqrt{\prowi\pcolj}}$, for some $i,j$, and so
$\spnorm{Q_t} \leq \max_{ij}\frac{1}{\sqrt{\prowi\pcolj}}\doteq R$.

The matrix $Q_t Q_t^T \in \R^{n \times n}$ is equal to
$\frac{e_{i,i}}{\prowi\pcolj}$ with probability $\pij$.
Hence $\EE{Q_t^T Q_t}$ is a diagonal matrix with entries $\sum_j
\frac{\pij}{\prowi\pcolj}$.  Similar arguments apply to
$Q_t^T Q_t$.  Multiplying by $s$, and recalling the spectral norm of a
diagonal matrix is simply the maximal magnitude element, we have:
\begin{equation*}
  \sigma^2 = s \cdot \max\left\{\max_i \sum_j
    \frac{\pij}{\prowi\pcolj},\max_j\sum_i\frac{\pij}{\prowi\pcolj}\right\}\;.
\end{equation*}
This completes the proof for the case that $\pp$ has uniform row- and column- marginals.

 Next we turn to the case that $\pp$ is a product distribution, $\pp=\pprow\times \ppcol$ (with possibly non-uniform marginals). For any $X\in\Wcal$, define
$$Z(X)=\left(X_{ij}\indicator{\pij\geq\frac{\log(n)}{s\sqrt{nm}}}\right)_{ij}\;.$$
Let  $\Zcal=\left\{Z(X):X\in\Wcal\right\}$.

We can then follow the proof of the bound in the uniform-marginals case, with a modified definition of $Q_t$:
 $$Q_t = \sigma_t \frac{e_{i_t,j_t}\indicator{\p{i_t}{j_t}\geq\frac{\log(n)}{s\sqrt{nm}}}}{\sqrt{\prow{i_t}\pcol{j_t}}}\;.$$
Proceeding as in the proof for Theorem~\ref{SquareRootThm}, we obtain $R\leq \sqrt{\frac{s\sqrt{nm}}{\log(n)}}$ and $\sigma^2\leq sn$, and thus
$$\Ep{S\sim p}{\hat{\Rcal}_S(\Zcal)}= \mathbf{O}\left(\sqrt{\frac{rn\log(n)}{s}}\right)\;.$$
Therefore, by~\citep{BartlettMendelson}, \begin{align*}
&\EE{\sup_{X\in\Wcal}L_p(Z(X))-\hat{L}_S(Z(X))}\leq  \mathbf{O}\left(l\cdot \sqrt{\frac{rn\log(n)}{s}}\right)\;,\\
&\EE{\sup_{X\in\Wcal}\hat{L}_S(Z(X))-L_p(Z(X))}\leq  \mathbf{O}\left(l\cdot \sqrt{\frac{rn\log(n)}{s}}\right)\;.\end{align*}

Next, let $I=\left(\sqrt{\pij}\indicator{\pij<\frac{\log(n)}{s\sqrt{nm}}}\right)_{ij}$. For any matrix $M$, define
$$\wfrnorm{M}=\frnorm{\Diag{\pprow}^{\nicefrac{1}{2}}M\Diag{\ppcol}^{\nicefrac{1}{2}}}\;.$$
Now take any $M$ with $\wfrnorm{M}\leq 1$. Let $M'=\Diag{\pprow}^{\nicefrac{1}{2}}M\Diag{\ppcol}^{\nicefrac{1}{2}}$, then $\frnorm{M'}\leq 1$. We have
\begin{align*}
&\sum_{ij:\pij<\frac{\log(n)}{s\sqrt{nm}}}\pij M_{ij}=\sum_{ij} I_{ij}M'_{ij}=\langle I, M'\rangle\leq \frnorm{I}\cdot\frnorm{M'}\\
&\leq \frnorm{I}=\sqrt{\sum_{ij}\pij\indicator{\pij<\frac{\log(n)}{s\sqrt{nm}}}}\leq \sqrt{nm\cdot\frac{\log(n)}{s\sqrt{nm}}}=\sqrt{\frac{\sqrt{nm}\log(n)}{s}}\;.
\end{align*}

Since $\frnorm{M}\leq\trnorm{M}$ for any matrix $M$, we then have, for any $X\in\Wcal$, $\wfrnorm{X}\leq \wtrnorm{X}\leq \sqrt{r}$, and so
\begin{align*}&\left|L_p(X)-L_p(Z(X))\right|
=\left|\sum_{ij\not\in\Ical}\pij\left(\ell(X_{ij},Y_{ij})-\ell(0,Y_{ij})\right)\right|\\&\leq l\cdot \sum_{ij\not\in\Ical}\pij|X_{ij}|\leq \sqrt{\frac{l^2r\sqrt{nm}\log(n)}{s}}\;.\end{align*}

And, fixing some $X^*\in\Wcal$ such that $L_p(X^*)=\inf_{X\in\Wcal} L_p(X)$,
\begin{align*}
&\EE{\sup_{X\in\Wcal}\hat{L}_S(Z(X))-\hat{L}_S(X)}+\EE{\hat{L}_S(X^*)-\hat{L}_S(Z(X^*))}\\
&=\EE{\sup_{X\in\Wcal}\frac{1}{s}\sum_{t=1}^s \indicator{(i_t,j_t)\not\in\Ical}\left(\ell(0,Y_{i_tj_t})-\ell(X_{i_tj_t},Y_{i_tj_t})\right)}\\
& \ \ \ \ \ \ \ \ \ +\EE{\frac{1}{s}\sum_{t=1}^s \indicator{(i_t,j_t)\not\in\Ical}\left(\ell(X^*_{i_tj_t},Y_{i_tj_t})-\ell(0,Y_{i_tj_t})\right)}\\
&=\EE{\sup_{X\in\Wcal}\frac{1}{s}\sum_{t=1}^s \indicator{(i_t,j_t)\not\in\Ical}\left(\ell(X^*_{i_tj_t},Y_{i_tj_t})-\ell(X_{i_tj_t},Y_{i_tj_t})\right)}\\
&\leq \EE{\sup_{X\in\Wcal}\frac{1}{s}\sum_{t=1}^s \indicator{(i_t,j_t)\not\in\Ical}\ell(X^*_{i_tj_t},Y_{i_tj_t})}\leq l\cdot \EE{\frac{1}{s}\sum_{t=1}^s \indicator{(i_t,j_t)\not\in\Ical}|X^*_{i_tj_t}|}\\
&=l\cdot  \EE{ \indicator{(i_1,j_1)\not\in\Ical}|X^*_{i_1j_1}|}=l\cdot \sum_{ij\not\in\Ical}\pij|X^*_{ij}|\leq  \sqrt{\frac{l^2r\sqrt{nm}\log(n)}{s}}\\
\end{align*}

Then writing
$$L_p(\hat{X}_S)-L_p(X^*)=(L_p(\hat{X}_S)-L_p(Z(\hat{X}_S)))+(L_p(Z(\hat{X}_S))-\hat{L}_S(Z(\hat{X}_S)))+(\hat{L}_S(Z(\hat{X}_S))-\hat{L}_S(\hat{X}_S))$$
$$+(\hat{L}_S(\hat{X}_S)-\hat{L}_S(X^*))+(\hat{L}_S(X^*)-\hat{L}_S(Z(X^*)))+(\hat{L}_S(Z(X^*))-L_p(Z(X^*)))+(L_p(Z(X^*))-L_p(X^*))\;,$$
we obtain
$$\EE{L_p(\hat{X}_S)-L_p(X^*)}\leq \mathbf{O}\left(\sqrt{\frac{l^2 rn\log(n)}{s}}\right)\;.$$

\subsection{Proof of Theorem~\ref{CubeRootThm}}\label{app:CubeRootThmTrans}

Assume $\ell$ is $l$-Lipschitz and $b$-bounded, and $r\geq 1$. We will show that (for any $\pp$)
$$\Ep{S\sim p}{\hat{\Rcal}_S(\ell\circ\Wcal)}= \mathbf{O}\left((l+b)\cdot \sqrt[3]{\frac{rn\log(n)}{s}}\right)\;.$$
Given a sample $S$, define
$$T^0_S=\left\{t \ : \ \prow{i_t}\text{ or }\pcol{j_t}<\sqrt[3]{ \frac{l^2r\log(n)}{b^2sn^2}}\right\}\;, T^1_S=\{1,\dots,s\}\backslash T^0_S\;.$$
We have
  \begin{align*}
   & \empRad_S(\ell\circ \Wcal) = \Ep{\sigma\sim\{\pm1\}^s}{\sup_{\wtrnorm{X} \leq \sqrt{r}} \frac{1}{s} \sum_{t=1}^s \sigma_t\cdot \ell(X_{i_tj_t},Y_{i_tj_t})}\\
      &\leq \Ep{\sigma}{\sup_{\wtrnorm{X} \leq \sqrt{r}}\frac{1}{s} \sum_{t\in T^0_S} \sigma_t\cdot \ell(X_{i_tj_t},Y_{i_tj_t})}
      + \Ep{\sigma}{\sup_{\wtrnorm{X} \leq \sqrt{r}}
      \frac{1}{s} \sum_{t\in T^1_S} \sigma_t\cdot \ell(X_{i_tj_t},Y_{i_tj_t})} \\
   \end{align*}
Bounding the first term,
\begin{align*}
&\Ep{\sigma}{\sup_{\wtrnorm{X} \leq \sqrt{r}}      \frac{1}{s} \sum_{t\in T^0_S} \sigma_t\cdot \ell(X_{i_tj_t},Y_{i_tj_t})} \leq \Ep{\sigma}{\frac{1}{s} \sum_{t\in T^0_S} \left|\sigma_t\right|\cdot b}
=\frac{b}{s}\cdot \left|T^0_S\right|\;.\end{align*}
In expectation over $S$,
\begin{align*}
&\Ep{S}{\frac{b}{s}\cdot \left|T^0_S\right|}=b\cdot\Ep{ij\sim p}{\indicator{\prowi\text{ or }\pcolj<\sqrt[3]{ \frac{l^2r\log(n)}{b^2sn^2}}}}\\
&=b\cdot\sum_{ij}\pij\indicator{\prowi\text{ or } \pcolj< \sqrt[3]{ \frac{l^2r\log(n)}{b^2sn^2}}}\\
&\leq \left[ b\cdot\sum_{i:\prowi<\sqrt[3]{ \frac{l^2r\log(n)}{b^2sn^2}}}\sum_j \pij\right] +\left[b\cdot \sum_{j:\pcolj<\sqrt[3]{ \frac{l^2r\log(n)}{b^2sn^2}}}\sum_i \pij\right]\\
&=\left[ b\cdot\sum_{i:\prowi<\sqrt[3]{ \frac{l^2r\log(n)}{b^2sn^2}}}\prowi \right]+\left[ b\cdot\sum_{j:\pcolj<\sqrt[3]{ \frac{l^2r\log(n)}{b^2sn^2}}}\pcolj\right]\\
&\leq bn\cdot \sqrt[3]{ \frac{l^2r\log(n)}{b^2sn^2}}+ bm \sqrt[3]{ \frac{l^2r\log(n)}{b^2sn^2}}\leq 2\sqrt[3]{\frac{l^2brn\log(n)}{s}}\;.\end{align*}

To bound the second term, we use the fact that $\trnorm{\text{abs}(X)}\leq\trnorm{X}$ for any matrix $X$, where $\text{abs}(X)$ is the matrix defined via $\text{abs}(X)_{ij}=|X_{ij}|$. We have
\begin{align*}
&  \Ep{\sigma}{\sup_{\wtrnorm{X} \leq \sqrt{r}}\frac{1}{s} \sum_{t\in T^1_S} \sigma_t\cdot \ell(X_{i_tj_t},Y_{i_tj_t})}\\
& \leq \Ep{\sigma}{\sup_{\wtrnorm{X} \leq \sqrt{r}}\frac{1}{s} \sum_{t\in T^1_S} \sigma_t\cdot \left(\ell(X_{i_tj_t},Y_{i_tj_t})-\ell(0,Y_{i_tj_t}\right)}+\Ep{\sigma}{\sup_{\wtrnorm{X} \leq \sqrt{r}}\frac{1}{s} \sum_{t\in T^1_S} \sigma_t\cdot \ell(0,Y_{i_tj_t})}\\
& = \Ep{\sigma}{\sup_{\wtrnorm{X} \leq \sqrt{r}}\frac{1}{s} \sum_{t\in T^1_S} \sigma_t\cdot \left(\ell(X_{i_tj_t},Y_{i_tj_t})-\ell(0,Y_{i_tj_t}\right)}
 \leq l\cdot\Ep{\sigma}{\sup_{\wtrnorm{X} \leq \sqrt{r}}\frac{1}{s} \sum_{t\in T^1_S} \sigma_t\cdot |X_{i_tj_t}|} \\
&= l\cdot\Ep{\sigma}{\sup_{\trnorm{X'} \leq \sqrt{r}}\frac{1}{s} \sum_{t\in T^1_S} \frac{\sigma_t}{\sqrt{\prow{i_t}\pcol{j_t}}}\cdot |X'_{i_tj_t}|} \leq l\cdot\Ep{\sigma}{\sup_{\trnorm{X''} \leq \sqrt{r}}\frac{1}{s} \sum_{t\in T^1_S} \frac{\sigma_t}{\sqrt{\prow{i_t}\pcol{j_t}}}\cdot X''_{i_tj_t}} \\
&=l\sqrt{r}\cdot\Ep{\sigma}{\spnorm{\frac{1}{s} \sum_{t=1}^s\sigma_t\frac{e_{(i_t,j_t)}\indicator{\prow{i_t},\pcol{j_t}\geq\sqrt[3]{ \frac{l^2r\log(n)}{b^2sn^2}}}}{\sqrt{\prow{i_t}\pcol{j_t}}}}}\;,
\end{align*}
Defining $Q_t=\sigma_t\frac{e_{(i_t,j_t)}\indicator{\prow{i_t},\pcol{j_t}\geq\sqrt[3]{ \frac{l^2r\log(n)}{b^2sn^2}}}}{\sqrt{\prow{i_t}\pcol{j_t}}}$, we can follow identical arguments as in the proof of the first bound of this theorem. We have
$$\spnorm{Q_t}\leq \max_{ij}\frac{\indicator{\prowi,\pcolj\geq\sqrt[3]{ \frac{l^2r\log(n)}{b^2sn^2}}}}{\sqrt{\prowi\pcolj}}\leq \sqrt[3]{\frac{b^2sn^2}{l^2r\log(n)}}\doteq R\;,$$
and
\begin{align*}
&\sigma^2 \doteq
\max\left\{\spnorm{\sum\EE{Q_t^T Q_t}},\spnorm{\sum\EE{Q_t
      Q^T_t}} \right\}\\
      &\leq s  \cdot \max\left\{\max_i \sum_j
    \frac{\pij\indicator{\prowi,\pcolj\geq\sqrt[3]{ \frac{l^2r\log(n)}{b^2sn^2}}}}{\prowi\pcolj},\right.\\
    &\left.\hspace{5cm} \max_j\sum_i\frac{\pij\indicator{\prowi,\pcolj\geq\sqrt[3]{ \frac{l^2r\log(n)}{b^2sn^2}}}}{\prowi\pcolj}\right\}\\
      &\leq s\cdot \sqrt[3]{ \frac{b^2sn^2}{l^2r\log(n)}}  \cdot \max\left\{\max_i \sum_j
    \frac{\pij}{\prowi},\max_j\sum_i\frac{\pij}{\prowi}\right\}=\sqrt[3]{ \frac{b^2s^4n^2}{l^2r\log(n)}}\\
      \end{align*}
Then applying~\citep{TroppTailBounds}, we get
\begin{align*}
& \Ep{\sigma}{\sup_{\wtrnorm{X} \leq \sqrt{r}}\frac{1}{s} \sum_{t\in T^1_S} \sigma_t\cdot \ell(X_{i_tj_t},Y_{i_tj_t})}=\frac{l\sqrt{r}}{s}\Ep{S,\sigma}{\spnorm{\sum_{t=1}^sQ_t}}\\
&\leq\mathbf{O}\left(\frac{l\sqrt{r}}{s}\left(\sigma\sqrt{\log(n)}+R\log(n)\right)\right)\\
& \leq\mathbf{O}\left( \frac{l\sqrt{r}}{s}\left(\sqrt[6]{ \frac{b^2s^4n^2}{l^2r\log(n)}}\sqrt{\log(n)}+\sqrt[3]{\frac{b^2sn^2}{l^2r\log(n)}}\log(n)\right)\right)\\
&\leq\mathbf{O}\left(l^{\nicefrac{2}{3}}b^{\nicefrac{1}{3}}\sqrt[3]{\frac{rn\log(n)}{s}}+l^{\nicefrac{1}{3}}b^{\nicefrac{2}{3}}\left(\sqrt[3]{\frac{rn\log(n)}{s}}\right)^2\right)\;.\end{align*}
If $s\geq rn\log(n)$, then this proves the bound. If not, then the result is trivial, since $L_p(X)\leq b$ for any $X$.

\subsection{Proof of Theorem~\ref{EmpiricalThm}}

Throughout this section, assume $s\geq 24n\log(n)$. (If this is not the case, then we only need to prove excess error $\leq \mathbf{O}(l\sqrt{r})$, which is trivial given the class $\cWcal$.)
We also assume $s\leq \mathbf{O}\left(nm\log(nm)\right)$. (If this is not the case, then with high probability, we observe all entries of the matrix and obtain optimal recovery.)
The lemmas which are cited in this proof, are proved below.

Define
$$X^*=\arg\min_{X\in\tWcal}L_p(X), \ \ \ \ \  r^*=\wtrnorm{X^*}^2\leq r\;.$$

For any sample $S$, define
$$c(S)=\max\left\{0,\cwtrnorm{\frac{1}{\sqrt{r^*}}X^*}-1\right\}\;.$$
Then, for a fixed $S$,
$$\cwtrnorm{(1-c(S))X^*}=\sqrt{r^*}(1-c(S))\twtrnorm{\frac{1}{\sqrt{r^*}}X^*}\leq \sqrt{r} \ \Rightarrow \ (1-c(S))X^*\in\cWcal\;.$$

Applying Lemma~\ref{p_hat_sup} and Theorem~\ref{SmoothingThm},
\begin{align*}
&\EE{L_p(\hat{X}_S)-\hat{L}_S(\hat{X}_S)}
\leq\EE{\sup_{X\in{\cWcal}}\left(L_p(X)-\hat{L}_S(X)\right)}\\
&\leq \EE{\sup_{X\in2\cdot{\tWcal}}\left(L_p(X)-\hat{L}_S(X)\right)}+ \frac{8\sqrt{l^2rnm}}{n^2}
\leq \mathbf{O}\left(\sqrt{\frac{l^2rn\log(n)}{s}}\right)+ \frac{8\sqrt{l^2rnm}}{n^2}\\
&\leq \mathbf{O}\left(\sqrt{\frac{l^2rn\log(n)}{s}}\right)
\;.\end{align*}
And, similarly,
\begin{align*}
&\EE{\hat{L}_S((1-c(S))X^*)-L_p((1-c(S))X^*)}
\leq \EE{\sup_{X\in{\cWcal}}\left(\hat{L}_S(X)-L_p(X)\right)}\\
&\leq \EE{\sup_{X\in2\cdot{\tWcal}}\left(\hat{L}_S(X)-L_p(X)\right)}+ \frac{8\sqrt{l^2rnm}}{n^2}
\leq \mathbf{O}\left(\sqrt{\frac{l^2rn\log(n)}{s}}\right)+ \frac{8\sqrt{l^2rnm}}{n^2}\\
&\leq \mathbf{O}\left(\sqrt{\frac{l^2rn\log(n)}{s}}\right)
\;.\end{align*}

By definition, since $(1-c(S))X^*\in\cWcal$,
$$\EE{\hat{L}_S(\hat{X}_S)-\hat{L}_S((1-c(S))X^*)}\leq 0\;.$$

Finally, by Lemma~\ref{cS_lemma},
$$\EE{L_p((1-c(S))X^*)-L_p(X^*)}\leq \sqrt{\frac{2l^2rn}{s}}\;.$$

Combining all of the above, we get
$$\EE{L_p(\hat{X}_S)-L_p(X^*)}\leq O\left(\sqrt{\frac{l^2rn\log(n)}{s}}\right)\;.$$

\subsubsection{Lemmas for Theorem 6}

\begin{lemma}\label{p_hat_sup}
$$\EE{\sup_{X\in{\cWcal}}\left(L_p(X)-\hat{L}_S(X)\right)}\leq \EE{\sup_{X\in2\cdot{\tWcal}}\left(L_p(X)-\hat{L}_S(X)\right)}+ \frac{8\sqrt{l^2rnm}}{n^2}\;.$$
$$\EE{\sup_{X\in{\cWcal}}\left(\hat{L}_S(X)-L_p(X)\right)}\leq \EE{\sup_{X\in2\cdot{\tWcal}}\left(\hat{L}_S(X)-L_p(X)\right)}+ \frac{8\sqrt{l^2rnm}}{n^2}\;.$$

\end{lemma}
\begin{proof}
By Lemma~\ref{p_phat}, with probability at least $1-2n^{-2}$, for all $i,j$,
$$\cprowi\geq\frac{1}{2}\tprowi, \ \cpcolj\geq\frac{1}{2}\tpcolj\;.$$
Let $A$ be the event that these inequalities hold. If $A$ occurs, then for any $X\in\cWcal$,
\begin{align*}
\wtrnorm{X}&=\trnorm{\Diag{\tprowi}^{\nicefrac{1}{2}}X\Diag{\tpcolj}^{\nicefrac{1}{2}}}\\
&=\trnorm{\Diag{\frac{\tprowi}{\cprowi}}^{\nicefrac{1}{2}}\Diag{\cprowi}^{\nicefrac{1}{2}}X\Diag{\cpcolj}^{\nicefrac{1}{2}}\Diag{\frac{\tpcolj}{\cpcolj}}^{\nicefrac{1}{2}}}\\
&\leq 2\trnorm{\Diag{\cprowi}^{\nicefrac{1}{2}}X\Diag{\cpcolj}^{\nicefrac{1}{2}}}=2\cwtrnorm{X}\leq 2\sqrt{r}\;.\end{align*}
In this case, $\cWcal\subset2\cdot\tWcal$, and therefore,
\begin{align*}
&\sup_{X\in\cWcal} \left[\left(L_p(X)-L_p(\mathbf{0}_{n\times m})\right)-\left(\hat{L}_S(X)-\hat{L}_S(\mathbf{0}_{n\times m})\right)\right]\\
&\leq
\sup_{X\in2\cdot \tWcal} \left[\left(L_p(X)-L_p(\mathbf{0}_{n\times m})\right)-\left(\hat{L}_S(X)-\hat{L}_S(\mathbf{0}_{n\times m})\right)\right]\;.\end{align*}

Next we consider the case that $A$ does not occur. For any $X\in\cWcal$,
\begin{align*}
|X|_{\infty}&\leq \frnorm{X}=2\sqrt{nm}\frnorm{\Diag{\frac{1}{2n}\mathbf{1}_n}^{\nicefrac{1}{2}} X \Diag{\frac{1}{2m}\mathbf{1}_m}^{\nicefrac{1}{2}}}\\
&\leq 2\sqrt{nm}\trnorm{\Diag{\frac{1}{2n}\mathbf{1}_n}^{\nicefrac{1}{2}} X \Diag{\frac{1}{2m}\mathbf{1}_m}^{\nicefrac{1}{2}}}\leq2\sqrt{nm}\cwtrnorm{X}\leq 2\sqrt{rnm}\;.\end{align*}
Therefore,\begin{align*}
&\sup_{X\in\cWcal} \left[\left(L_p(X)-L_p(\mathbf{0}_{n\times m})\right)-\left(\hat{L}_S(X)-\hat{L}_S(\mathbf{0}_{n\times m})\right)\right]\\
&\leq \sup_{X\in\cWcal} \left[\sum_{ij}\pij\left(\ell(X_{ij},Y_{ij})-\ell(0,Y_{ij})\right)-\frac{1}{s}\sum_t\left(\ell(X_{i_tj_t},Y_{i_tj_t})-\ell(0,Y_{i_tj_t})\right)\right]\\
&\leq l\cdot \sup_{X\in\cWcal} \left[\sum_{ij}\pij\cdot|X_{ij}|+\frac{1}{s}\sum_t|X_{i_tj_t}|\right]\\
&\leq l\cdot \sup_{X\in\cWcal} \left[\sum_{ij}\pij\cdot2\sqrt{rnm}+\frac{1}{s}\sum_t2\sqrt{rnm}\right]\leq 4\sqrt{l^2rnm}\\
\end{align*}

And so,
\begin{align*}
&\Ep{S}{\sup_{X\in\cWcal} \left[\left(L_p(X)-L_p(\mathbf{0}_{n\times m})\right)-\left(\hat{L}_S(X)-\hat{L}_S(\mathbf{0}_{n\times m})\right)\right]}\\
&=\Ep{S}{\sup_{X\in\cWcal} \left[\left(L_p(X)-L_p(\mathbf{0}_{n\times m})\right)-\left(\hat{L}_S(X)-\hat{L}_S(\mathbf{0}_{n\times m})\right)\right]\cdot\indicator{A}}\\
& \ \ \ \ \ \ + \Ep{S}{\sup_{X\in\cWcal} \left[\left(L_p(X)-L_p(\mathbf{0}_{n\times m})\right)-\left(\hat{L}_S(X)-\hat{L}_S(\mathbf{0}_{n\times m})\right)\right]\cdot\indicator{A^c}}\\
&\leq \Ep{S}{\sup_{X\in2\cdot \tWcal} \left[\left(L_p(X)-L_p(\mathbf{0}_{n\times m})\right)-\left(\hat{L}_S(X)-\hat{L}_S(\mathbf{0}_{n\times m})\right)\right]\cdot\indicator{A}}  + P\left(A^c\right)\cdot 4\sqrt{l^2rnm}\\
&\leq \Ep{S}{\sup_{X\in2\cdot \tWcal} \left[\left(L_p(X)-L_p(\mathbf{0}_{n\times m})\right)-\left(\hat{L}_S(X)-\hat{L}_S(\mathbf{0}_{n\times m})\right)\right]\cdot\indicator{A}}  + \frac{8\sqrt{l^2rnm}}{n^2}\\
&\leq \Ep{S}{\sup_{X\in2\cdot \tWcal} \left[\left(L_p(X)-L_p(\mathbf{0}_{n\times m})\right)-\left(\hat{L}_S(X)-\hat{L}_S(\mathbf{0}_{n\times m})\right)\right]}  +  \frac{8\sqrt{l^2rnm}}{n^2}\;.
\end{align*}
where the last step is true because, since $\mathbf{0}_{n\times m}\in 2\cdot\tWcal$, for any $S$,
$$\sup_{X\in2\cdot \tWcal} \left[\left(L_p(X)-L_p(\mathbf{0}_{n\times m})\right)-\left(\hat{L}_S(X)-\hat{L}_S(\mathbf{0}_{n\times m})\right)\right]\geq 0\;.$$
And, $\Ep{S}{L_p(\mathbf{0}_{n\times m})-\hat{L}_S(\mathbf{0}_{n\times m})}=0$, so therefore,
$$\Ep{S}{\sup_{X\in{\cWcal}}\left(L_p(X)-\hat{L}_S(X)\right)}\leq \EE{\sup_{X\in2\cdot{\tWcal}}\left(L_p(X)-\hat{L}_S(X)\right)}+ \frac{8\sqrt{l^2rnm}}{n^2}\;.$$
The second claim can be proved with identical arguments.
\end{proof}

\begin{lemma}\label{p_phat}
 With probability at least $1-2n^{-2}$, for all $i$ and all $j$,
$$\cprowi\geq\frac{1}{2}\tprowi, \ \cpcolj\geq\frac{1}{2}\tpcolj\;.$$
\end{lemma}
\begin{proof}

Take any row $i$. Suppose that $\prowi\leq\frac{1}{n}$. Then $\tprowi\leq\frac{1}{n}$, while $\cprowi=\frac{1}{2}\left(\hprowi+\frac{1}{n}\right)\geq\frac{1}{2n}$. Therefore, in this case, $\cprowi\geq\frac{1}{2}\tprowi$ with probability $1$.

Next, suppose that $\prowi>\frac{1}{n}$. Then, by the Chernoff inequality,
\begin{align*}
P\left(\hprowi<\frac{1}{2}\prowi\right)
&=P\left(\mathrm{Bin}(s,\prowi)<s\prowi\left(1-\frac{1}{2}\right)\right)\leq e^{-\frac{s\prowi}{8}}\\
&\leq e^{-\frac{s}{8n}}\leq e^{-3\log(n)}=n^{-3}\;.\end{align*}
Therefore, with probability at least $1-n^{-3}$, $\hprowi\geq \frac{1}{2} \prowi$, and so
\begin{align*}
\cprowi&=\frac{1}{2}\left(\hprowi+\frac{1}{n}\right)\geq \frac{1}{2}\left(\frac{1}{2}\prowi+\frac{1}{n}\right)\geq \frac{1}{2}\tprowi\;.\end{align*}

Therefore, for any row $i$, with probability at least $1-n^{-3}$, $\cprowi\geq\frac{1}{2}\tprowi$. The same reasoning applies to every column $j$. Therefore, with probability at least $1-2n^{-2}$, the statement holds for all $i$ and all $j$.
\end{proof}

\begin{lemma}\label{cS_lemma} Fix $X^*$ with $\twtrnorm{X^*}^2=r^*\leq r$, and define
$$c(S)=\max\left\{0,\cwtrnorm{\frac{1}{\sqrt{r^*}}X^*}-1\right\}\;.$$
Then
$$\EE{L_p((1-c(S))X^*)-L_p(X^*)}\leq \sqrt{\frac{2l^2rn}{s}}\;.$$
\end{lemma}
\begin{proof}
\begin{align*}
&{L_p((1-c(S))X^*)-L_p(X^*)}=\sum_{ij}\pij\left(\ell((1-c(S))X^*_{ij},Y_{ij})-\ell(X^*_{ij},Y_{ij})\right)\\
&\leq l\cdot \sum_{ij}\pij |(1-c(S))X^*_{ij}-X^*_{ij}|= l\cdot c(S)\cdot \sum_{ij}\pij |X^*_{ij}|\\
&= l\cdot c(S)\cdot \sum_{ij}\frac{\pij}{\sqrt{\tprowi\tpcolj}} \cdot\sqrt{\tprowi\tpcolj}\cdot|X^*_{ij}|\\
\end{align*}
Defining $M=\left(\frac{\pij}{\sqrt{\tprowi\tpcolj}} \right)_{ij}$,
\begin{align*}
&=  l\cdot c(S)\cdot \langle M, \left(\Diag{\tprowi}^{\nicefrac{1}{2}}X^*\Diag{\tpcolj}^{\nicefrac{1}{2}}\right)_{ij}\rangle\\
&\leq   l\cdot c(S)\cdot \spnorm{M}\cdot\trnorm{\left(\Diag{\tprowi}^{\nicefrac{1}{2}}X^*\Diag{\tpcolj}^{\nicefrac{1}{2}}\right)_{ij}}\\
&\leq   l\sqrt{r}\cdot c(S)\cdot \spnorm{M}\;.
\end{align*}

Now we show that $\spnorm{M}\leq 2$. Take any unit vectors $u\in\R^m$, $v\in\R^n$. Then
\begin{align*}
&u^TMv=\sum_{ij}\pij\cdot\sqrt{\frac{u_i^2}{\tprowi}}\cdot\sqrt{\frac{v_j^2}{\tpcolj}}\leq\frac{1}{2}\sum_{ij}\pij \left(\frac{u_i^2}{\tprowi}+\frac{v_j^2}{\tpcolj}\right)\\
&=\frac{1}{2}\sum_i \prowi\cdot\frac{u_i^2}{\tprowi}+\frac{1}{2}\sum_j \pcolj\cdot\frac{v_j^2}{\tpcolj}\leq \frac{1}{2}\sum_i 2 u_i^2+\frac{1}{2}\sum_j 2 v_j^2=2\;.\end{align*}
So, by Lemma~\ref{ExpectedTraceWt},
\begin{align*}
&\EE{L_p((1-c(S))X^*)-L_p(X^*)}\leq 2l\sqrt{r}\cdot \EE{c(S)}\leq 2l\sqrt{r}\cdot\sqrt{\frac{n}{2s}}\;.\end{align*}

\end{proof}

\begin{lemma}\label{ExpectedTraceWt}
For any $\pp$, for any fixed $X$ with $\twtrnorm{X}=1$,
$$\EE{\max\{0,\cwtrnorm{X}-1\}}\leq \sqrt{\frac{n}{2s}}\;.$$
\end{lemma}

\begin{proof} By properties of the trace-norm~\citep{ShraibmanSrebro}, we can write $\Diag{\tpprow}^{\nicefrac{1}{2}}X\Diag{\tppcol}^{\nicefrac{1}{2}}=AB^T$, where $\frnorm{A}^2=\frnorm{B}^2=\wtrnorm{X}=1$.
Define
$$D_1=\Diag{\cpprow}\Diag{\tpprow}^{-1}, \ D_2=\Diag{\cppcol}\Diag{\tppcol}^{-1}\;.$$
Then, by properties of the trace-norm~\citep{ShraibmanSrebro},
\begin{align*}
&\cwtrnorm{X}=\trnorm{\Diag{\cpprow}^{\nicefrac{1}{2}}X\Diag{\cppcol}^{\nicefrac{1}{2}}}=\trnorm{\left(D_1^{\nicefrac{1}{2}}A\right)\left(D_2^{\nicefrac{1}{2}}B\right)^T}\\
&\leq\frac{1}{2}\frnorm{D_1^{\nicefrac{1}{2}}A}^2+\frac{1}{2}\frnorm{D_2^{\nicefrac{1}{2}}B}^2\\
&=\frac{1}{2}\sum_i \frac{\cprowi}{\tprowi}\|A_{(i)}\|_2^2+\frac{1}{2}\sum_j\frac{\cpcolj}{\tpcolj}\|B_{(j)}\|^2_2\\
&=\frac{1}{4}\sum_i \frac{\hprowi+\frac{1}{n}}{\tprowi}\|A_{(i)}\|_2^2+\frac{1}{4}\sum_j\frac{\hpcolj+\frac{1}{m}}{\tpcolj}\|B_{(j)}\|^2_2\\
&=\frac{1}{4}\sum_i \frac{N^r_i+\frac{s}{n}}{s\tprowi}\|A_{(i)}\|_2^2+\frac{1}{4}\sum_j\frac{N^c_j+\frac{s}{m}}{s\tpcolj}\|B_{(j)}\|^2_2\;,
\end{align*}
where $N^r_i$ is the number of samples in row $i$, and $N^c_j$ is the number of samples in column $j$. Clearly,
\begin{align*}
&\EE{\frac{1}{4}\sum_i \frac{N^r_i+\frac{s}{n}}{s\tprowi}\|A_{(i)}\|_2^2+\frac{1}{4}\sum_j\frac{N^c_j+\frac{s}{m}}{s\tpcolj}\|B_{(j)}\|^2_2}\\
&=\frac{1}{4}\sum_i \frac{s\prowi+\frac{s}{n}}{s\tprowi}\|A_{(i)}\|_2^2+\frac{1}{4}\sum_j\frac{s\pcolj+\frac{s}{m}}{s\tpcolj}\|B_{(j)}\|^2_2\\
&=\frac{1}{4}\sum_i \frac{2s\tprowi}{s\tprowi}\|A_{(i)}\|_2^2+\frac{1}{4}\sum_j\frac{2s\tpcolj}{s\tpcolj}\|B_{(j)}\|^2_2\\
&=\frac{1}{2}\frnorm{A}^2+\frac{1}{2}\frnorm{B}^2=1\;.
\end{align*}

And, we can compute
$$\text{Var}(N^r_i)\leq s\prowi,  \ \text{Cov}(N^r_i,N^r_{i'})<0, \ \text{Var}(N^c_j)\leq s\pcolj, \ \text{Cov}(N^c_j,N^c_{j'})<0\;.$$
Therefore,
\begin{align*}
&\text{Var}\left(\sum_i \frac{N^r_i}{s\tprowi}\|A_{(i)}\|^2_2+\sum_j \frac{N^c_j}{s\tpcolj}\|B_{(j)}\|^2_2\right)\\
&\leq 2\text{Var}\left(\sum_i \frac{N^r_i}{s\tprowi}\|A_{(i)}\|^2_2\right)+2\text{Var}\left(\sum_j \frac{N^c_j}{s\tpcolj}\|B_{(j)}\|^2_2\right)\\
&=\sum_i \frac{1}{s^2\tprowi^2}\text{Var}(N^r_i)\|A_{(i)}\|^4_2+2\sum_{i<i'}\frac{1}{s^2\tprowi\tprow{i'}}\text{Cov}(N^r_i,N^r_{i'})\|A_{(i)}\|^2_2\|A_{(i')}\|^2_2\\
& \ \ \ \ + \sum_j \frac{1}{s^2\tpcolj}\text{Var}(N^c_j)\|B_{(j)}\|^4_2+2\sum_{j<j'}\frac{1}{s^2\tpcolj\tpcol{j'}}\text{Cov}(N^c_j,N^c_{j'})\|B_{(j)}\|^2_2\|B_{(j')}\|^2_2\\
&\leq \sum_i \frac{1}{s^2\tprowi^2}\text{Var}(N^r_i)\|A_{(i)}\|^4_2+ \sum_j \frac{1}{s^2\tpcolj}\text{Var}(N^c_j)\|B_{(j)}\|^4_2\\
&\leq \sum_i \frac{s\prowi}{s^2\tprowi^2}\|A_{(i)}\|^4_2+ \sum_j \frac{s\pcolj}{s^2\tpcolj}\|B_{(j)}\|^4_2\\
\end{align*}
Since $\tprowi\geq\frac{1}{2}\prowi$ and $\tprowi\geq\frac{1}{2n}$, and similarly for the columns, we continue:
\begin{align*}
&\leq  \sum_i \frac{4n}{s}\|A_{(i)}\|^4_2+ \sum_j \frac{4m}{s}\|B_{(j)}\|^4_2\leq \frac{4n}{s}\left( \sum_i \|A_{(i)}\|^2_2\right)^2+ \frac{4m}{s}\left(\sum_j \|B_{(j)}\|^2_2\right)\\
&\leq \frac{4n}{s}\frnorm{A}^4+ \frac{4m}{s}\frnorm{B}^4\leq \frac{4(n+m)}{s}\;.\end{align*}
So, we have\begin{align*}
&\EE{\max\{0,\cwtrnorm{X}-1\}}\\
&\leq \EE{\max\left\{0,\frac{1}{4}\sum_i \frac{N^r_i+\frac{s}{n}}{s\tprowi}\|A_{(i)}\|_2^2+\frac{1}{4}\sum_j\frac{N^c_j+\frac{s}{m}}{s\tpcolj}\|B_{(j)}\|^2_2-1\right\}}\\
&\leq\sqrt{\text{Var}\left(\frac{1}{4}\sum_i \frac{N^r_i+\frac{s}{n}}{s\tprowi}\|A_{(i)}\|_2^2+\frac{1}{4}\sum_j\frac{N^c_j+\frac{s}{m}}{s\tpcolj}\|B_{(j)}\|^2_2\right)}\\
&=\sqrt{\text{Var}\left(\frac{1}{4}\sum_i \frac{N^r_i}{s\tprowi}\|A_{(i)}\|_2^2+\frac{1}{4}\sum_j\frac{N^c_j}{s\tpcolj}\|B_{(j)}\|^2_2\right)}
\leq\sqrt{\frac{(n+m)}{4s}}\end{align*}
\end{proof}

\section{Proofs for the transductive setting}\label{AppendixTransductive}

\subsection{Proof of Theorem~\ref{EmpiricalImprovementThm}}

Let $\overline{S}\subset[n]\times [m]$ be a subset of size $2s$. Let $\opp$ denote the smoothed empirical marginals of $\overline{S}$.

Now choose any $S\subset\overline{S}$, a training set of size $s$. Without loss of generality, write $\overline{S}=\left\{(i_1,j_1),\dots,(i_{2s},j_{2s})\right\}$ and $S=\left\{(i_1,j_1),\dots,(i_{s},j_{s})\right\}$.

First, we bound transductive Rademacher complexity. By Lemma 12 in~\citep{ShraibmanSrebro}, for any sample $S$,
\begin{align*}
\hat{\Rcal}_S(\oWcal)&=\Ep{\sigma\sim\{\pm1\}^s}{\sup_{X\in\oWcal}\frac{1}{s}\sum_{t=1}^s\sigma_tX_{i_tj_t}}\\
&=\Ep{\sigma\sim\{\pm1\}^s}{\sup_{X\in\oWcal}\frac{1}{s}\sum_{ij}X_{ij}\left(\sum_{t:(i_t,j_t)=(i,j)}\sigma_t\right)}\\
&\leq \Ep{\sigma\sim\{\pm1\}^{n\times m}}{\sup_{X\in\oWcal}\frac{1}{s}\sum_{ij}X_{ij}\sigma_{ij}\cdot \#\{t:(i_t,j_t)=(i,j),1\leq t\leq s\}}\\
&= \Ep{\sigma\sim\{\pm1\}^{n\times m}}{\sup_{X\in\oWcal}\frac{1}{s}\sum_{ij}X_{ij}\sigma_{ij}\cdot \indicator{(i,j)\in S}}\;.\end{align*}
Now define matrix $\Sigma$ via
$$\Sigma_{ij}=\frac{\indicator{(i,j)\in S}}{s\sqrt{\oprowi\opcolj}}\;.$$
We have
\begin{align*}\Ep{S\sim p}{\hat{\Rcal}_S(\oWcal)}
&\leq \Ep{S}{\Ep{\sigma\sim\{\pm1\}^{n\times m}}{\sup_{X\in\oWcal}\frac{1}{s}\sum_{ij}X_{ij}\sigma_{ij}\cdot \indicator{(i,j)\in S}}}\\
&=\Ep{S}{\Ep{\sigma\sim\{\pm1\}^{n\times m}}{\sup_{X\in\oWcal}\sum_{ij}\left(\sqrt{\prowi}X_{ij}\sqrt{\pcolj}\right)\sigma_{ij} \Sigma_{ij}}}\\
&=\Ep{S}{\Ep{\sigma\sim\{\pm1\}^{n\times m}}{\sup_{X:\trnorm{X}\leq \sqrt{r}}X_{ij}\sigma_{ij}\Sigma_{ij}}}\\
&=\sqrt{r}\cdot\Ep{S}{\Ep{\sigma\sim\{\pm1\}^{n\times m}}{\spnorm{\sigma\bullet \Sigma}}}\;,
\end{align*}
where $\sigma\bullet\Sigma$ is the element-wise product of $\Sigma$ with the random sign matrix $\sigma=(\sigma_{ij})$. By~\citep{Seginer},
$$\Ep{\sigma\sim\{\pm1\}^{n\times m}}{\spnorm{\sigma\bullet \Sigma}}\leq \mathbf{O}\left( \log^{1/4}(n+m)\right)\cdot\max\left\{\max_i \norm{\Sigma_{(i)}}_2,\max_j \norm{\Sigma^{(j)}}_2\right\}\;.$$
We now bound $\norm{\Sigma_{(i)}}_2$ and $\norm{\Sigma^{(j)}}_2$. Fix any $i$. Then
\begin{align*}
&\norm{\Sigma_{(i)}}_2^2
=\sum_j \Sigma_{ij}^2=\sum_{j=1}^m \frac{\indicator{(i,j)\in S}}{\left(s\oprowi\right)\cdot\left( s\opcolj\right)}\leq \sum_{j=1}^m \frac{\indicator{(i,j)\in \overline{S}}}{\left(s\oprowi\right)\cdot\left( s\cdot\frac{1}{2m}\right)}\\
&\leq  \sum_{j=1}^m \frac{\#\{t:(i_t,j_t)=(i,j),1\leq t\leq 2s\}}{\left(\frac{1}{4}\#\left\{t:i_t=i,1\leq t\leq 2s\right\}\right)\cdot\left(s\cdot\frac{1}{2m}\right)}\leq   \frac{\#\{t:i_t=i,1\leq t\leq 2s\}}{\left(\frac{1}{4}\#\left\{t:i_t=i,1\leq t\leq 2s\right\}\right)\cdot\left(s\cdot\frac{1}{2m}\right)}\leq \frac{8m}{s}\;.\end{align*}
 Similarly, for all $j$, $\norm{\Sigma^{(j)}}_2^2\leq\frac{8n}{s}$. Therefore,
 \begin{align*}&\Ep{S\sim p}{\hat{\Rcal}_S(\oWcal)}
\leq \sqrt{r}\cdot\Ep{S}{\Ep{\sigma\sim\{\pm1\}^{n\times m}}{\spnorm{\sigma\bullet \Sigma}}}\\
&\leq\sqrt{r}\cdot\Ep{S}{\mathbf{O}\left( \log^{1/4}(n)\right)\cdot\max\left\{\max_i \norm{\Sigma_{(i)}}_2,\max_j \norm{\Sigma^{(j)}}_2\right\}}
\leq \mathbf{O}\left(\sqrt{\frac{rn\log^{\nicefrac{1}{2}}(n)}{s}}\right)\;.
\end{align*}

Applying Theorem 5 of~\citep{S_SS_COLT} (using integration to obtain a bound in expectation from a bound in probability),
$$\Ep{S}{\hat{L}_{\overline{S}\backslash S}(\hat{X}_S)- \inf_{X\in\oWcal}\hat{L}_{\overline{S}\backslash S} (X)}\leq\mathbf{O}\left(\sqrt{\frac{l^2rn\log^{\nicefrac{1}{2}}(n)+b^2}{s}}\right)\;.$$

\subsection{Transductive version of Theorem~\ref{SquareRootThm}}

Let $\opp$ now denote the (unsmoothed) empirical marginals of $\overline{S}$. If $\oprowi\geq \frac{1}{Cn}$ and $\opcolj\geq \frac{1}{Cm}$ for all $i,j$, defining
$$\hat{X}_S=\arg\min_{X\in\oWcal}\hat{L}_S(X)\;,$$
we can then show that, for an $l$-Lipschitz loss $\ell$ bounded by $b$, in expectation over the split of $\overline{S}$ into training set $S$ and test set $T$,
$$\hat{L}_T(\hat{X}_S)\leq \inf_{X\in\oWcal}\hat{L}_T(X)+\mathbf{O}\left(C^{\nicefrac{1}{2}}l\cdot \sqrt{\frac{rn\log^{\nicefrac{1}{2}}(n)+b^2}{s}}\right)\;.$$

We prove this by following identical arguments as in the proof of Theorem~\ref{EmpiricalImprovementThm}, we define
$$\Sigma_{ij}=\frac{\indicator{(i,j)\in S}}{s\sqrt{\oprowi\opcolj}}\;,$$
and obtain $\norm{\Sigma_{(i)}}_2^2,\norm{\Sigma^{(j)}}_2^2\leq\frac{2Cn}{s}$ for all $i,j$, which yields
$$\Ep{S}{\hat{L}_{\overline{S}\backslash S}(\hat{X}_S)- \inf_{X\in\oWcal}\hat{L}_{\overline{S}\backslash S} (X)}\leq\mathbf{O}\left(\sqrt{\frac{Cl^2rn\log^{\nicefrac{1}{2}}(n)+b^2}{s}}\right)\;.$$
In fact, we can obtain the same result with a weaker requirement on $\opp$, namely
\begin{align*}&\frac{s}{n}\max\left\{\max_i\|\Sigma_{(i)}\|^2_2,\max_j\|\Sigma^{(j)}\|^2_2\right\}\leq \max\left\{\max_i \frac{1}{m}\sum_{j=1}^m \frac{\frac{1}{s}\indicator{(i,j)\in \overline{S}}}{\oprowi\opcolj},\max_j \frac{1}{n}\sum_{i=1}^n \frac{\frac{1}{s}\indicator{(i,j)\in \overline{S}}}{\oprowi\opcolj}\right\}\leq C\;.\end{align*}
For instance, this quantity is likely to be bounded if $\overline{S}$ is a sample drawn from a product distribution on the matrix.

\subsection{Transductive version of Theorem~\ref{CubeRootThm}}
Let $\opp$ now denote the (unsmoothed) empirical marginals of $\overline{S}$. We define
$$\hat{X}_S=\arg\min_{X\in\oWcal}\hat{L}_S(X)\;,$$
we can then show that, for an $l$-Lipschitz loss $\ell$ bounded by $b$, without any requirements on $\opp$, in expectation over the split of $\overline{S}$ into training set $S$ and test set $T$,
$$\hat{L}_T(\hat{X}_S)\leq \inf_{X\in\oWcal}\hat{L}_T(X)+\mathbf{O}\left((l+b)\cdot \sqrt[3]{\frac{rn\log(n)}{s}}\right)\;.$$

We prove this by combining the proof techniques used in the proofs of Theorems~\ref{CubeRootThm} and~\ref{EmpiricalImprovementThm}. Define
$$T^0_S=\left\{t \ : \ 1\leq t\leq 2s, \prow{i_t}\text{ or }\pcol{j_t}<\sqrt[3]{ \frac{l^2r\log(n)}{b^2sn^2}}\right\}\;, T^1_S=\{1,\dots,2s\}\backslash T^0_S\;.$$
We then have
\begin{align*}
&\hat{\Rcal}_S(\ell\circ \oWcal)=\Ep{\sigma\sim\{\pm1\}^s}{\sup_{X\in\oWcal}\frac{1}{s}\sum_{t=1}^s\sigma_t\ell(X_{i_tj_t},Y_{i_tj_t})}\\
&\leq \Ep{\sigma\sim\{\pm1\}^{n\times m}}{\sup_{X\in\oWcal}\frac{1}{s}\sum_{ij}\ell(X_{i_tj_t},Y_{i_tj_t})\sigma_{ij}\cdot \indicator{(i,j)\in S}}\\
&\leq \Ep{\sigma}{\sup_{X\in\oWcal}\frac{1}{s}\sum_{ij}\ell(X_{i_tj_t},Y_{i_tj_t})\sigma_{ij}\cdot \indicator{(i,j)\in S,\text{ and }\oprowi,\opcolj\geq \sqrt[3]{ \frac{l^2r\log(n)}{b^2sn^2}}}}\\
& \ \ \ \ \ +\Ep{\sigma}{\sup_{X\in\oWcal}\frac{1}{s}\sum_{ij}\ell(X_{i_tj_t},Y_{i_tj_t})\sigma_{ij}\cdot \indicator{(i,j)\in S,\text{ and }\oprowi\text{ or }\opcolj<\sqrt[3]{ \frac{l^2r\log(n)}{b^2sn^2}}}}\\
&\doteq \text{(Term 1)} + \text{(Term 2)};.\end{align*}
Now define matrix $\Sigma$ via
$$\Sigma_{ij}=\frac{\indicator{(i,j)\in S,\text{ and }\oprowi,\opcolj\geq \sqrt[3]{ \frac{l^2r\log(n)}{b^2sn^2}}}}{s\sqrt{\oprowi\opcolj}}\;.$$
Following the same arguments as in the proof of Theorem~\ref{EmpiricalImprovementThm}, we obtain for all $i,j$,
$$\|\Sigma_{(i)}\|^2_2,\|\Sigma^{(j)}\|^2_2\leq \frac{4}{s}\cdot \sqrt[3]{ \frac{b^2sn^2}{l^2r\log(n)}}$$
Therefore, using the same arguments as in the proof of Theorem~\ref{CubeRootThm},
$$\text{(Term 1)}\leq l\sqrt{r}\mathbf{O}\left(\log^{\nicefrac{1}{4}}(n)\sqrt{ \frac{4}{s}\cdot \sqrt[3]{ \frac{b^2sn^2}{l^2r\log(n)}}}\right)=\mathbf{O}\left(\sqrt[3]{\frac{l^2brn\log(n)}{s}}\right)\;.$$

Next we have
\begin{align*}
&\text{(Term 2)}=\Ep{\sigma}{\sup_{X\in\oWcal}\frac{1}{s}\sum_{ij}\ell(X_{i_tj_t},Y_{i_tj_t})\sigma_{ij}\cdot \indicator{(i,j)\in S,\text{ and }\oprowi\text{ or }\opcolj<\sqrt[3]{ \frac{l^2r\log(n)}{b^2sn^2}}}}\\
&\leq \sup_{X\in\oWcal}\frac{1}{s}\sum_{ij}\ell(X_{i_tj_t},Y_{i_tj_t})\cdot \indicator{(i,j)\in S,\text{ and }\oprowi\text{ or }\opcolj<\sqrt[3]{ \frac{l^2r\log(n)}{b^2sn^2}}}\\
&\leq\frac{1}{s}\sum_{ij}b\cdot \indicator{(i,j)\in \overline{S},\text{ and }\oprowi\text{ or }\opcolj<\sqrt[3]{ \frac{l^2r\log(n)}{b^2sn^2}}}\\
&\leq \frac{1}{s}\sum_{i:\oprowi<\sqrt[3]{ \frac{l^2r\log(n)}{b^2sn^2}}} \left(\sum_{j:(i,j)\in \overline{S}} b\right) + \frac{1}{s}\sum_{j:\opcolj<\sqrt[3]{ \frac{l^2r\log(n)}{b^2sn^2}}} \left(\sum_{i:(i,j)\in \overline{S}} b\right)\\
&\leq \frac{2n}{s}\left(b\cdot 2s\cdot\sqrt[3]{ \frac{l^2r\log(n)}{b^2sn^2}}\right) \leq \mathbf{O}\left(\sqrt[3]{\frac{l^2rbn\log(n)}{s}}\right)\;.\\
\end{align*}

Combining the two, we get
$\Rcal_s(\ell\circ\oWcal)\leq  \mathbf{O}\left(\sqrt[3]{\frac{l^2rbn\log(n)}{s}}\right)$,
and therefore, in expectation over the split of $\overline{S}$ into $S$ and $T$,
$$\hat{L}_T(\hat{X}_S)\leq \inf_{X\in\oWcal}\hat{L}_T(X)+ \mathbf{O}\left((l+b)\cdot \sqrt[3]{\frac{rn\log(n)}{s}}\right)\;.$$

\end{document}